\documentclass[twoside,11pt]{article}
\usepackage[preprint]{jmlr2e}

\usepackage{tikz}
\usetikzlibrary{positioning,arrows.meta,calc,fit,decorations.pathreplacing}
\usepackage{amsmath}

\usepackage{graphicx}
\usepackage{subcaption} 

\usepackage{algorithm}
\usepackage{algpseudocode}

\NewDocumentCommand{\bang}
{ mO{} }{\textcolor{red}{\textsuperscript{\textit{Bang}}\textsf{\textbf{\small[#1]}}}}

\ShortHeadings{Budgeted Multi-Agent Synergy}{Bang Liu, Linglong Kong, Jian Pei}
\firstpageno{1}

\begin{document}

\title{Phase Transition for Budgeted Multi-Agent Synergy}

\author{\name Bang Liu \email bang.liu@umontreal.ca \\
       \addr Department of Computer Science and Operations Research\\
       Université de Montréal \& Mila - Quebec AI Institute\\
       Montreal, QC, Canada
       \AND
       \name Linglong Kong \email lkong@ualberta.ca \\
       \addr Department of Mathematical \& Statistical Science\\
       University of Alberta, AB, Canada
       \AND
       \name Jian Pei \email j.pei@duke.edu \\
       \addr Department of Computer Science\\
       Duke University, NC, United States
       }

\editor{My editor}

\maketitle

\begin{abstract}
Multi-agent systems can improve reliability, yet under a fixed test-time budget they often help, saturate, or even collapse. We develop a deliberately minimal and \emph{calibratable} theory that predicts these regimes from three binding constraints of modern agent stacks: finite context windows, lossy inter-agent communication, and shared failures among similar agents. Each leaf agent is summarized by a compute--performance scaling exponent $\beta$; communication is captured by a message-length fidelity curve $\gamma(m)$; dependence is captured by an effective shared-error correlation $\rho$; and a context window $W$ imposes hard fan-in limits that make hierarchy necessary. For binary success/failure tasks with majority aggregation, we prove a sharp phase transition for deep $b$-ary trees with correlated inputs and lossy communication: a single scalar $\alpha_\rho$ (combining $\gamma(m)$, $\rho$, and fan-in $b$) determines whether weak signal is amplified to a nontrivial fixed point or washed out to chance. In the amplifying regime, we derive an \emph{organization exponent} $s$ and show that \emph{budgeted synergy}, i.e., outperforming the best single agent under the same total budget, occurs exactly when $s>\beta$, yielding closed-form compute allocation rules and explicit budget thresholds. We further characterize saturation via a mixing depth and provide a conservative clipped predictor that remains accurate across growth and saturation. A continuous-performance warm-up gives closed-form risks for star, chain, and tree organizations, making correlation- and communication-induced floors explicit and exposing the core design trade-offs in a smooth setting. Finally, we validate the predicted phase boundaries in controlled synthetic simulations and show how the same mechanisms explain the dominant bottlenecks reported in recent large-scale matched-budget studies of LLM agent-system scaling, including context saturation, subcritical error cascades, and diminishing returns at strong baselines.
\end{abstract}


\section{Introduction}
\label{sec:intro}

Multi-agent systems are often presented as a reliability primitive: if a single agent is fallible, run many agents, let them interact, and aggregate their outputs.
Under a \emph{fixed computational budget per task}, however, multi-agent coordination is not reliably beneficial.
The same additional agents that sometimes improve performance can also lead to saturation or even outright degradation.
This paper develops a \emph{predictive} theory of this brittleness: when multi-agent organization produces \emph{budgeted synergy}, i.e., outperforming the best single agent under the same total computational budget, and when scaling out must fail.

The LLM era has made this phenomenon impossible to ignore.
It is now routine to instantiate multiple LLM-based agents, assign roles, and coordinate them through discussion, critique, or voting in so-called ``agent societies'' and orchestration frameworks \citep{brown2020language,yao2023react,wu2023autogen,li2023camel,chen2024agentverse,hong2023metagpt}.
The underlying intuition is classical: ensembling can reduce error \citep{breiman1996bagging,freund1997decision,dietterich2000ensemble}, and interaction protocols such as debate or amplification aim to turn many weak judgments into a stronger one \citep{irving2018debate,christiano2018amplifying,du2024multiagentdebate}.
At the same time, controlled matched-budget evaluations increasingly reveal strong task- and topology-dependence, including regimes in which multi-agent variants \emph{degrade} once coordination overhead and information loss dominate \citep{kim2025scalingagents}.

Our starting point is that these failures are not implementation accidents but structural consequences of a small number of constraints that repeatedly bind in modern agent stacks.
First, agent errors are not independent: multiple instances derived from the same base model and prompt often share systematic failure modes, producing ``groupthink'' rather than error cancellation.
Second, coordination is communication-limited: as systems scale, agents must compress reasoning into bounded-length messages, and the resulting information loss compounds across layers.
Third, coordination is context-limited: a central aggregator cannot read arbitrarily many messages under a finite context window, so naive scale-out saturates structurally.
Recent work on large-population collaboration and topology learning highlights the same bottleneck: token and context limits are frequently the binding resource \citep{qian2025scalingcollaboration,zhang2024gdesigner}.

Today, most agent-system design still proceeds by heuristics: choose a topology, message format, and number of agents, then iterate.
In spirit, this resembles the hand-crafted feature engineering era of machine learning: powerful components exist, but system-level behavior depends on ad hoc choices and expensive trial-and-error.
Our goal is to push multi-agent design toward an ``agentic intelligence from first principles'' approach: make the dominant constraints explicit, explain \emph{why} scaling out works or fails under a fixed computational budget, and derive sharp, quantitative regime boundaries, most notably, a phase transition for budgeted multi-agent synergy.

\paragraph{A minimal, calibratable coordination theory.}
We develop a deliberately small framework that isolates the coordination bottlenecks above while remaining analytically tractable.
The term \emph{agent} is used abstractly: a black-box solver that, given a task instance and an allocated compute budget, produces an output that another component can consume (a decision, a score, or an $m$-token message).
This deliberately includes LLM-based agents as a primary motivating instance, because (i) tokens provide a natural budget unit, (ii) finite context windows induce hard fan-in constraints, and (iii) single-agent capability is often well described by compute scaling laws \citep{kaplan2020scaling,hoffmann2022computeoptimal}.
At the same time, the abstraction is broader than LLMs: none of the results relies on language as such, only on (a) how single-agent quality scales with compute, (b) how much usable information survives bounded-length communication, and (c) how strongly agents' failures co-move.

\begin{figure}[!ht]
  \centering
  \includegraphics[width=\linewidth]{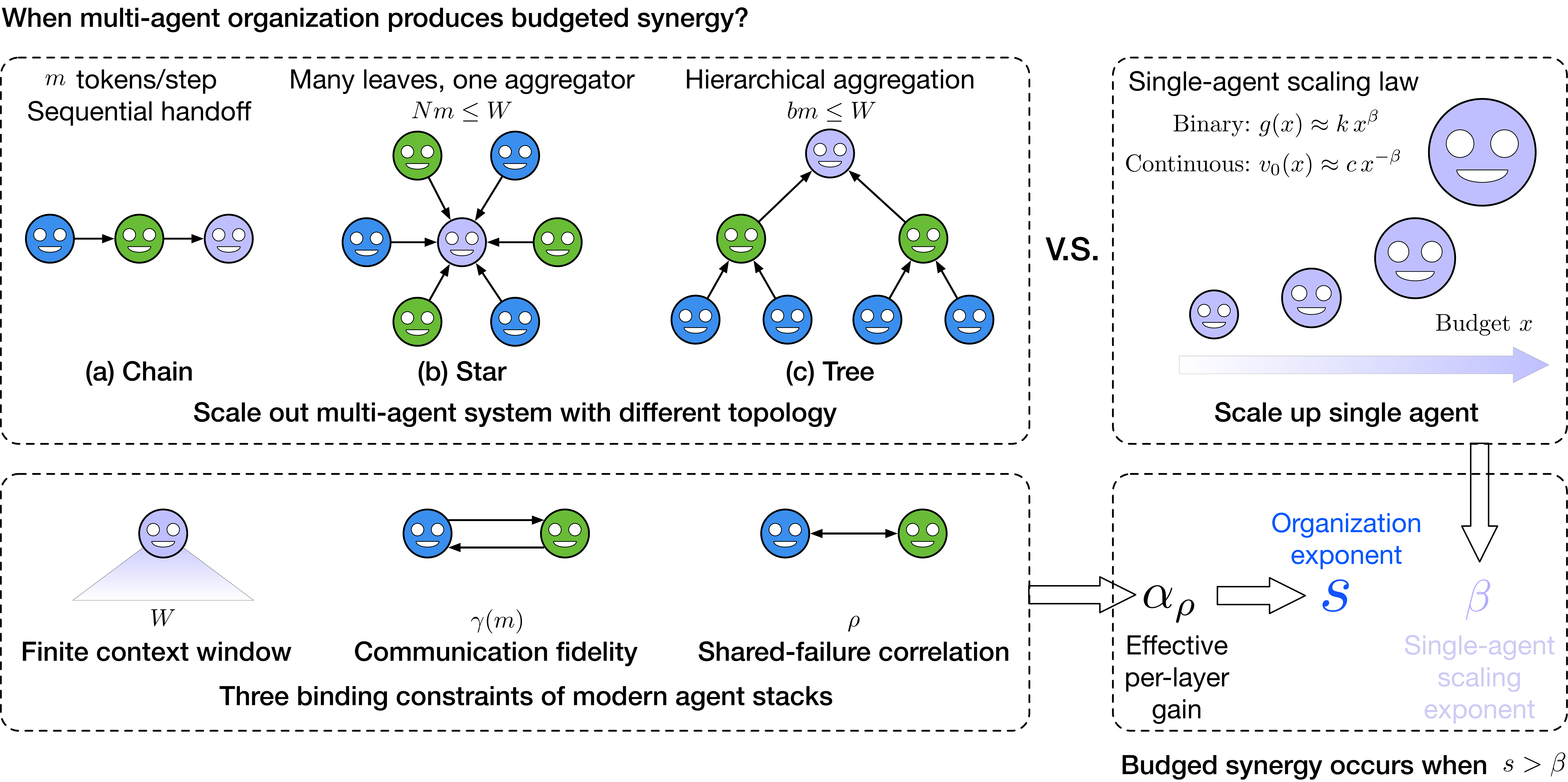}
  \caption{When does multi-agent organization produce budgeted synergy?
  An overview of our theoretical framework and core contribution.
  Left: scaling out multi-agent systems under different topologies (chain, star, and hierarchical tree) subject to a finite context window $W$, which constrains fan-in globally for stars ($Nm \le W$) and locally for trees ($bm \le W$).
  Bottom: three binding constraints of modern agent stacks: finite context windows, lossy communication with fidelity $\gamma(m)$, and shared-failure correlation $\rho$.
  Right: the competing baseline of scaling up a single agent, governed by a single-agent scaling law with exponent $\beta$.
  Our theory combines these elements into an effective per-layer gain $\alpha_\rho$, which induces an organization exponent $s$ for hierarchical aggregation.
  \textbf{Budgeted synergy occurs precisely when $s > \beta$}, meaning that scaling out via organization can outperform scaling up a single agent under the same total budget.}
  \label{fig:overview}
\end{figure}

By \emph{calibratable} we do not mean a theory ``about measurement.''
We mean a coordination theory parameterized by a small set of \emph{environment-dependent effective quantities} that can be estimated once (via targeted experiments) and then used to predict qualitative regimes.
Concretely, the framework is driven by single-agent scaling (exponent $\beta$), communication fidelity as a function of message length (e.g., $\gamma(m)$), shared-failure correlation ($\rho$), and the context window constraint ($W$), together with a total budget $B$ that fixes the scale-up/scale-out trade-off.
These parameters are not claimed to capture every detail of a real agent stack; they form a minimal interface for diagnosing which constraint is binding and for predicting when organization helps, saturates, or fails.

\paragraph{Phase transition for budgeted multi-agent synergy.}
Within this framework we analyze three canonical organizations: star (central aggregation), chain (sequential handoff), and hierarchical tree.
Our main results characterize when hierarchy is viable and when it is provably counterproductive.
For binary success/failure tasks with majority aggregation, we prove a sharp amplification--collapse transition for deep trees: a scalar $\alpha_\rho$ (combining communication fidelity, correlation, and fan-in) determines whether weak signal amplifies across layers or is washed out to chance, in the spirit of reconstruction thresholds in noisy broadcasting on trees \citep{evans2000broadcasting,mossel2001reconstruction,kesten1966additional}.
In the amplifying regime, we derive an \emph{organization exponent} $s$ governing growth with the number of leaves and show that \emph{budgeted} synergy in the growth range occurs exactly when $s>\beta$, yielding closed-form compute allocation and explicit budget thresholds.
We also characterize finite-depth saturation via a nontrivial fixed point and mixing-depth guarantees.
Finally, we make precise why topology becomes decisive only because constraints bind: absent context limits, central aggregation is information-dominant (data processing inequality) \citep{cover2006elements}, and hierarchy is valuable only insofar as it bypasses fan-in limits without crossing into the collapsing regime. Figure~\ref{fig:overview} illustrates our core contribution.

Beyond the binary core, a continuous-performance warm-up yields closed-form risks for star/chain/tree organizations and makes correlation and communication floors explicit.
We also translate theorems into design diagnostics (e.g., monotone communication design curves) and connect the predicted bottlenecks to recent controlled large-scale agent-scaling studies \citep{kim2025scalingagents}.
Detailed calibration templates are collected in Appendix~\ref{app:calibration}.

\paragraph{Methodological stance: theory vs.\ benchmarking.}
Our aim in this paper is to theoretically delineate regimes of behavior and identify the mechanisms that govern when multi-agent systems help, saturate, or collapse under realistic constraints.
Large-scale matched-budget evaluations with real LLM agent stacks are expensive and can be confounded by prompt engineering and rapidly changing model/tool ecosystems.
Accordingly, we validate the sharp regime boundaries via controlled synthetic simulations (Section~\ref{sec:synthetic_experiments}) that probe the theory under its own assumptions, and we use external controlled evidence \citep{kim2025scalingagents} to show that the bottlenecks isolated by the theory (context saturation, subcritical collapse, and diminishing returns) are salient in modern deployments.

\paragraph{Organization.}
Section~\ref{sec:framework} defines the framework, measurable parameters, and the budget and context constraints.
Section~\ref{sec:continuous_warmup} develops the continuous warm-up.
Section~\ref{sec:binary_core} presents the binary theory, including the phase transition, small-signal amplification, and saturation; it also compares topologies under context constraints and derives budget thresholds.
Section~\ref{sec:design} turns the theory into design diagnostics.
Section~\ref{sec:experiments} provides empirical touchpoints via external controlled evidence and synthetic sanity checks.
We place detailed proofs and optional calibration/evaluation templates in the appendices.

\section{Framework and Measurable Parameters}
\label{sec:framework}

This section formalizes the abstraction used throughout the paper.
We keep only three constraints that repeatedly bind in practice: finite context windows, lossy communication, and shared failures, and summarize the environment by a small set of quantities.
These quantities play distinct roles: \(\beta\) captures how a single agent improves with compute, \(\gamma(m)\) (or \(\sigma_c^2(m)\)) captures what survives an \(m\)-token message, and \(\rho\) captures how correlated different agents' errors are.
Together with a total budget \(B\) and a context window \(W\), they determine whether scaling out amplifies signal, saturates at a floor, or collapses.

We model a multi-agent system as a directed acyclic computation graph.
Leaf nodes (agents) produce signals about a latent task variable \(Y\); internal nodes receive messages from their children (possibly through a lossy channel) and apply an aggregation rule; the root output is the system prediction.
Throughout, an \emph{agent} is a black-box solver with an adjustable compute budget \(x\) that emits a bounded-length output (a decision, a score, or an \(m\)-token message).
LLM-based agents are the primary motivating instance, but the framework only assumes a single-agent scaling curve, a communication-fidelity curve, and a shared-failure correlation.

\subsection{Tasks and performance metrics}
\label{sec:tasks-metrics}

We study two task primitives that are deliberately simple but theoretically expressive.
They are not meant to exhaust the diversity of real-world tasks; rather, they isolate the coordination bottlenecks that arise repeatedly inside more complex multi-agent pipelines.
Binary success/failure captures decision-like outcomes (e.g., accept/reject, pass/fail, correctness of a proposed action), while continuous performance on \([0,1]\) captures graded judgments, scores, or calibrated estimates.
Together, these two primitives cover a wide range of evaluation signals used in practice and allow clean analysis of aggregation, communication loss, and shared failures. From a theoretical standpoint, any richer task that can be decomposed into local scoring, communication, and aggregation steps must confront the same coordination constraints analyzed here.

The continuous setting serves as a smooth warm-up where many quantities admit closed forms and correlation and communication floors are explicit.
The binary setting is not a special case of the continuous one: the nonlinearity of majority vote introduces qualitatively new behavior, including a sharp amplification--collapse phase transition that has no analogue in linear aggregation.
Many realistic agent workflows combine both primitives. For example, agents may exchange graded scores or confidence estimates internally, but the final system decision is binary. So analyzing both is necessary to capture the full coordination picture.

\paragraph{Binary tasks.}
Let \(Y \in \{ -1, +1 \}\) denote the ground truth.
A system outputs \(\widehat{Y}\in\{-1,+1\}\).
We measure performance through the \emph{bias}
\begin{equation}
\mu \;:=\; \mathbb{E}[\widehat{Y}Y] \in [-1,1],
\qquad
\Pr(\widehat{Y}=Y)=\frac{1+\mu}{2}.
\label{eq:bias_def}
\end{equation}
The bias \(\mu\) is a natural summary statistic for weak agents: it measures how far performance is above chance and composes cleanly under simple channel models and aggregation maps.
This makes it particularly suitable for analyzing whether small, local improvements are amplified or destroyed by hierarchical organization.

\paragraph{Continuous tasks.}
Let \(Y\in[0,1]\) and a system output \(\widehat{Y}\in[0,1]\).
We use mean squared error (MSE)
\begin{equation}
v \;:=\; \mathbb{E}\big[(\widehat{Y}-Y)^2\big],
\label{eq:mse_def}
\end{equation}
and occasionally report a bounded performance score such as \(\mathrm{Perf}=1/(1+v)\), which is monotone in \(v\).
Continuous metrics capture graded information that is often exchanged within agent systems (scores, confidence levels, value estimates) and lead to linear aggregation dynamics.
In this setting, the effect of correlation and communication loss appears as explicit error floors rather than phase transitions, providing intuition that complements the binary analysis.
\subsection{Single-agent capability and scaling}
\label{sec:scaling}

Our central comparison is \emph{fixed budget}: the same total computational budget \(B\) can be allocated either to \textbf{scale up} a single agent (making one agent stronger) or to \textbf{scale out} across many agents (running more agents in parallel).
Understanding when scale out can outperform scale up under the same budget is the core question of this paper.

We model per-leaf compute by a single nonnegative knob \(x\ge 0\), representing the share of the per-task computational budget allocated to an individual agent.
In LLM-based agents, \(x\) may correspond to inference tokens devoted to reasoning, the number of samples or self-consistency votes, structured test-time inference procedures (e.g., generate--select or tournament-style algorithms), the number of tool calls, or any other additive per-task resource that trades off directly against parallelism.

We summarize the effect of per-agent compute via a \emph{single-agent scaling law}.
Classical scaling laws are typically reported for training loss as a function of training compute \citep{kaplan2020scaling,hoffmann2022computeoptimal}.
Here, however, we use scaling laws in a different but related sense: as an \emph{effective description} of how a fixed trained agent’s \emph{test-time performance} improves when allocated additional per-task computational resources.
Recent work on test-time or inference-time scaling shows that such improvements can follow systematic power-law or exponential trends in appropriate regimes and metrics, including accuracy, error probability, or calibrated scores \citep{chen2024provable,levi2025simple,snell2024scaling,wu2025inference}.
Motivated by these findings, we treat the scaling exponent as a \emph{per-task scaling exponent} governing scale up, not as a statement about retraining the model.

For binary success/failure tasks, a leaf agent allocated compute \(x\) produces a vote \(\widehat{Y}\) with bias
\[
\mu_0(x)=g(x),
\]
where \(g(\cdot)\) is increasing and exhibits diminishing returns.
For continuous tasks, a leaf produces an estimate \(X\) with conditional variance \(v_0(x)\), where \(v_0(\cdot)\) decreases with \(x\).

In our main theorems, we work in a regime where these functions are well-approximated by power laws:
\begin{align}
\text{Binary:}\quad & g(x) \approx k\,x^{\beta} \quad \text{(small-signal regime)},\label{eq:binary_scaling}\\
\text{Continuous:}\quad & v_0(x) \approx c\,x^{-\beta}, \label{eq:cont_scaling}
\end{align}
with constants \(k,c>0\) and a \emph{single-agent scaling exponent} \(\beta>0\).
Here ``small-signal regime'' refers to the \emph{operating range of \(x\)} in which (i) the local power-law approximation is accurate enough for comparison, and (ii) the induced leaf performance is not already saturated.
In the binary analysis, this typically corresponds to weak leaf votes (small bias \(\mu_0(x)\), i.e., accuracy only slightly above chance), which is precisely the regime where the majority update map is well-approximated by its derivative at the origin and where the organization exponent \(s\) governs growth.
In the continuous warm-up, it corresponds to the pre-saturation range where aggregation has not yet hit the communication/correlation floor.
The exponent \(\beta\) is not a universal constant: it depends on the model, prompting, tools, test-time procedure, and task family, and is intended to be \emph{measured} in the operating regime of interest.
Our theory does not require the power-law approximation to hold globally or asymptotically; it relies only on the existence of a local scaling exponent \(\beta\) in the budget range where the scale-up versus scale-out comparison is made.

\subsection{Communication as a controllable lossy channel}
\label{sec:communication}

Agents coordinate through messages.
Our analysis uses a compact abstraction of how message length affects what downstream nodes can reliably use.

We model each edge as a channel controlled by a message length parameter \(m\) (in tokens).
Longer messages generally preserve more usable information, but they also cost more budget and reduce feasible fan-in under a finite context window.

\paragraph{Binary: an effective bit channel.}
For binary tasks and majority-style protocols, it is natural to treat each child as intending to transmit a single bit of information (its vote or decision), encoded into a message of length \(m\).
We model the receiver's decoded bit as passing through a binary symmetric channel (BSC) with \emph{reliability} \(\gamma(m)\in(0,1]\):
\begin{equation}
\mu_{\text{recv}} = \gamma(m)\,\mu.
\label{eq:binary_channel}
\end{equation}
Equivalently, \(\gamma(m)=1-2\epsilon(m)\), where \(\epsilon(m)\) is an effective flip probability.

\paragraph{Continuous: additive distortion.}
For continuous tasks, we model transmission as additive zero-mean distortion,
\begin{equation}
\widetilde{X}=X+\eta,
\qquad
\mathbb{E}[\eta]=0,\quad \mathrm{Var}(\eta)=\sigma_c^2(m),
\label{eq:cont_channel}
\end{equation}
where \(\sigma_c^2(m)\) decreases with \(m\).
This captures compression loss, imperfect interpretation, and degradation introduced by summarization or constrained prompts.

The intent is not to claim that real messages are literally bits or Gaussian perturbations, but to represent an \emph{effective} fidelity curve as a function of \(m\).
In particular, \(\gamma(m)\) or \(\sigma_c^2(m)\) can be estimated by simple encode--decode experiments at the same message length.

\subsection{Shared-failure correlation and the \texorpdfstring{$\rho$}{rho} model}
\label{sec:correlation}

A key departure from idealized ensemble analyses is that agent errors are not independent, even though many ensemble gains are easiest to analyze under weak dependence assumptions \citep{dietterich2000ensemble,germain2015risk}.
Agents instantiated from the same base model and prompt often share systematic error modes, which can dominate any aggregation gain.
We capture this effect with a single, measurable correlation parameter \(\rho\in[0,1)\).

\paragraph{Binary: correlation of signed correctness.}
For a leaf vote \(\widehat{Y}_i\), define the signed correctness variable
\[
S_i := \widehat{Y}_i Y \in \{-1,+1\}.
\]
Then \(\mathbb{E}[S_i]=\mu_0\), and we define
\[
\rho \;:=\; \mathrm{Corr}(S_i,S_j), \qquad i\neq j,
\]
estimated empirically by averaging pairwise correlations across tasks and agent instances.
Here \(i\) and \(j\) index two agent instances whose outputs are aggregated together: for a star this is any pair of leaves; for a tree it should be interpreted as the \emph{sibling} correlation at an internal node. In practice the effective correlation can depend on depth, heterogeneity, or the message schema; for tractability we treat \(\rho\) as a depth-independent effective parameter and return to extensions in Section~\ref{sec:discussion}.
Intuitively, \(\rho=0\) corresponds to independent errors, while \(\rho\) close to 1 indicates near-perfect groupthink.

For analysis, we use a concrete generative model that matches this definition exactly.
At each aggregation node, with probability \(\rho\) the children share a common ``mode'' (their signed correctness variables are identical), and with probability \(1-\rho\) they are conditionally independent given \(Y\).
This model produces the same pairwise correlation \(\rho\) while remaining analytically tractable and yields a closed-form correlated aggregation map used throughout the paper.

\paragraph{Continuous: correlated residuals.}
For continuous outputs \(X_i\), define residuals \(E_i := X_i - Y\).
We use the standard equal-correlation model
\[
\mathrm{Var}(E_i)=v,\qquad \mathrm{Cov}(E_i,E_j)=\rho v,\quad i\neq j,
\]
which makes correlation floors explicit and leads to closed-form recursions in Section~\ref{sec:continuous_warmup}.
The same \(\rho\) can be estimated by pairwise residual correlations.

\paragraph{Why a single parameter is useful.}
Real systems may have richer dependence structure than a single coefficient.
The intent here is not to model all dependence, but to isolate the dominant failure direction: shared errors.
As we show later, this single parameter already shifts phase boundaries and imposes aggregation floors that are otherwise difficult to diagnose.

\subsection{Budget and context constraints}
\label{sec:budget-context}

Our central comparison is \emph{fixed-budget}: a multi-agent system should be judged against the best single-agent baseline under the same total cost.
We therefore make computation and communication costs explicit.

\paragraph{Budget.}
Let \(B\) denote the total available budget.
In token-based LLM systems, a natural accounting unit is tokens, including both generated tokens and tokens processed as input.
Our theory is compatible with other budget units (time, API cost, calls), as long as they are additive and comparable across designs.

We separate two costs:
\begin{itemize}
  \item \textbf{Leaf computation:} each leaf agent is allocated compute \(x\).
  \item \textbf{Communication:} each edge transmits a message of length \(m\) tokens.
\end{itemize}
To keep the model simple, we account a per-edge cost proportional to \(m\).
A useful approximation in token-based systems is
\[
\text{edge cost} \approx 2m,
\]
representing \(m\) tokens produced by the sender and \(m\) tokens read by the receiver.

\paragraph{Context window as a fan-in constraint.}
Let \(W\) denote the maximum number of tokens that an internal node can read and process.
If each incoming message has length \(m\), then any aggregation node can read at most
\begin{equation}
\text{fan-in} \;\le\; \Big\lfloor \frac{W}{m} \Big\rfloor.
\label{eq:fanin_constraint}
\end{equation}
This inequality is the reason depth matters:
a star is capped by \(N\le \lfloor W/m\rfloor\) leaves, while a tree can scale to \(N=b^L\) leaves by distributing the fan-in constraint across levels.

\paragraph{Topology-specific cost summaries.}
Let \(N\) be the number of leaves.
Under the accounting above:
\begin{itemize}
  \item \textbf{Star:} \(C_{\text{star}}(N,x,m)\approx N(x+2m)\), with feasibility constraint \(Nm\le W\).
  \item \textbf{Chain:} for length \(L\), \(C_{\text{chain}}(L,x,m)\approx (L+1)x + 2mL\) (assuming each handoff step runs an agent with compute x, e.g., one model call), with feasibility constraint \(m\le W\) (usually nonbinding).
  \item \textbf{Tree:} for a full \(b\)-ary tree of depth \(L\) with \(N=b^L\) leaves, the number of edges is
  \[
  E = b\cdot \frac{N-1}{b-1},
  \]
  so \(C_{\text{tree}}(b,L,x,m)\approx Nx + 2mE\).
  It is convenient to write this as a per-leaf cost:
  \[
  C_{\text{tree}}(b,L,x,m)\approx N\big(x+c_0(b,m)\big),
  \qquad
  c_0(b,m):=2m\frac{b}{b-1},
  \]
  subject to the fan-in feasibility \(bm\le W\). Alternative accounting (templating overheads, parsing costs, asymmetric protocols) rescales \(c_0(b,m)\) and can be treated as an implementation constant.
\end{itemize}
This approximation isolates the basic trade-off: increasing \(m\) can improve fidelity, but it also increases coordination cost and reduces the number of leaves affordable under \(B\).

\subsection{Topologies and protocols}
\label{sec:topologies_protocols}

A \emph{topology} specifies how information flows; a \emph{protocol} specifies what each node does with the information it receives.
We keep protocols deliberately simple to make the effect of constraints and scaling transparent.
In particular, we focus on three canonical organizations that recur in practice and isolate the key trade-offs: star, chain, and hierarchical tree (Figure~\ref{fig:overview}).

\paragraph{Star.}
Leaves transmit their outputs directly to a central aggregator.
Star avoids multi-hop loss, but it concentrates fan-in at one node and therefore saturates at \(N \approx W/m\) leaves regardless of budget.

\paragraph{Chain.}
Outputs are passed sequentially through \(L\) steps.
Chains avoid fan-in bottlenecks but repeatedly retransmit information; without new evidence injected along the path, each step can only preserve or degrade the signal.

\paragraph{Hierarchical tree.}
A tree distributes aggregation across many small nodes.
Each internal node aggregates \(b\) child messages and forwards a summary upward, enforcing \(bm\le W\) locally and enabling many more leaves.
The cost is multi-hop communication and repeated aggregation, so performance hinges on whether amplification dominates compounding loss and correlation.

\paragraph{Why focus on star, chain, and tree?}
These motifs capture the dominant structural trade-offs induced by finite context and imperfect coordination.
Star minimizes depth but is capped by a global fan-in constraint \(Nm \le W\).
Chain removes fan-in limits but maximizes the number of lossy transmissions.
Trees are the simplest family that bypasses the fan-in cap (\(bm \le W\) locally) while still aggregating many signals; this is precisely the setting in which phase transitions and budget thresholds emerge.

\begin{figure}[!ht]
\centering
  \includegraphics[width=0.65\linewidth]{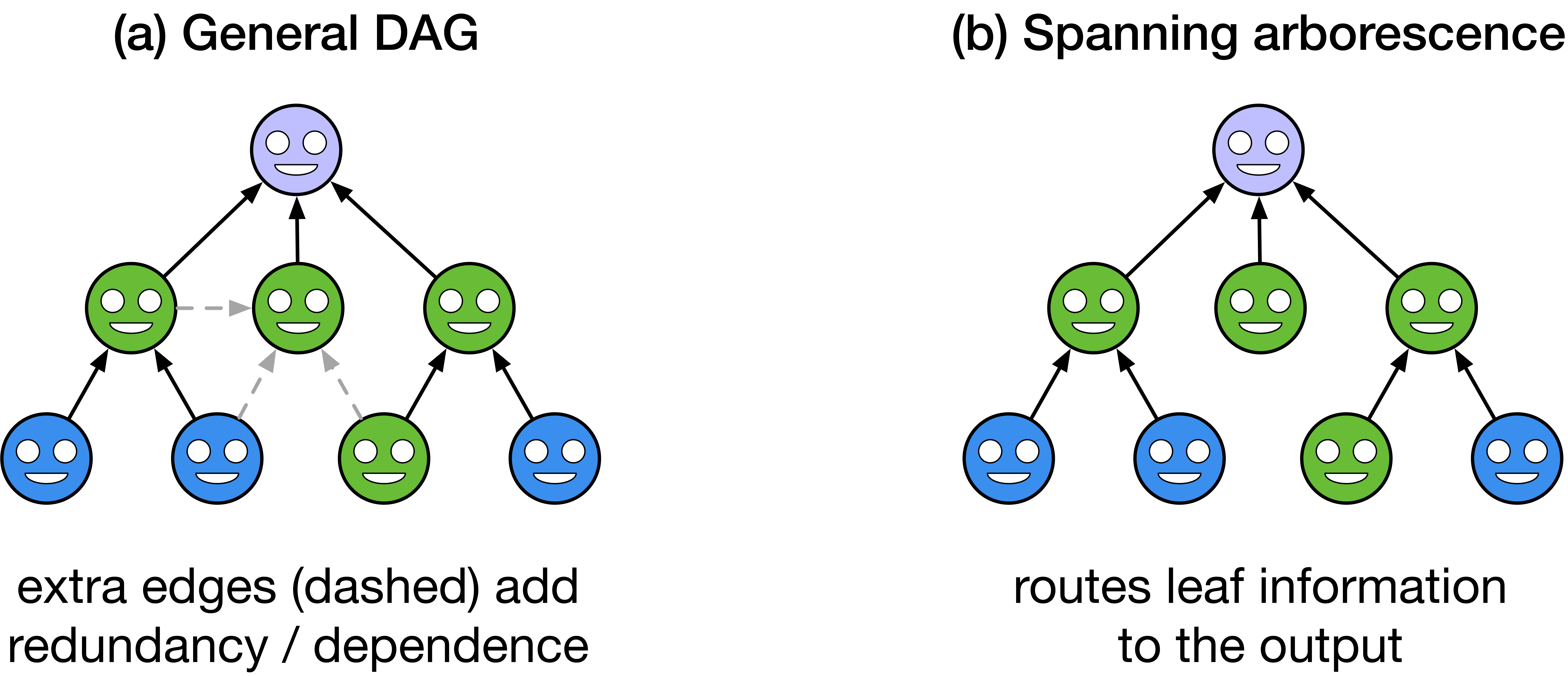}
\caption{From general DAGs to spanning arborescences.
  A general DAG~(a) may contain extra edges (dashed) that add redundancy or dependence,
  but any feed-forward organization contains a spanning arborescence that routes
  information from leaves to the output~(b).}
\label{fig:dag_to_arborescence}
\end{figure}

It is also useful to view more general feed-forward organizations through these motifs.
With a finite context window, every node can only aggregate a bounded number of incoming messages, i.e., each local computation is an \emph{in-star} with fan-in at most \(\lfloor W/m \rfloor\).
Globally, any connected directed acyclic communication graph contains a spanning arborescence that routes leaf information to the output; star and chain are the two extreme cases (depth \(1\) versus fan-in \(1\)), and bounded-fan-in trees are the minimal structure that scales beyond the star’s context cap.
Additional edges or multi-round interaction are valuable mainly when they effectively \emph{change the measurable parameters}, for example by improving effective fidelity through redundancy or by reducing effective correlation through independent verification, rather than by rerouting the same information along multiple dependent paths.
Figure~\ref{fig:dag_to_arborescence} illustrates this reduction from general DAGs to a spanning arborescence.

\paragraph{Aggregation rules.}
We analyze majority vote at internal nodes for binary tasks (ties are avoided by choosing odd fan-in) and averaging for continuous tasks (the MSE-optimal linear unbiased aggregator under equal-variance equal-correlation assumptions).
More sophisticated protocols (critique, verification, debate) can be interpreted in our framework as mechanisms that effectively increase communication fidelity or reduce shared-error correlation.
We return to this connection in the discussion.

\subsection{What is measurable and how to measure it}
\label{sec:measurable-params}

The framework is parameterized by a small set of environment-dependent quantities.
Table~\ref{tab:measurable_params} summarizes their meaning and outlines how one might estimate them in a specific model--task setting.
These estimates are optional: the theoretical results that follow hold given \(\beta\), \(\rho\), and the communication curves, while calibration is a way to connect the regime predictions back to concrete agent stacks at a coarse level.

\begin{table}[t]
\centering
\small
\renewcommand{\arraystretch}{1.15}
\begin{tabular}{l p{0.45\linewidth} p{0.33\linewidth}}
\hline
\textbf{Quantity} & \textbf{Meaning in the framework} & \textbf{How to estimate in practice} \\
\hline
\(\beta\) & Single-agent scaling exponent: how fast a single agent improves with compute \(x\). & Sweep \(x\) and fit a power law to bias \(g(x)\) (binary) or MSE \(v_0(x)\) (continuous) in the non-saturated regime. \\
\(\gamma(m)\) / \(\sigma_c^2(m)\) & Communication fidelity as a function of message length \(m\). Binary uses an effective bit reliability \(\gamma(m)\); continuous uses a distortion variance \(\sigma_c^2(m)\). & Encode known labels/values into \(m\) tokens and measure decode flip rate (binary) or MSE (continuous). \\
\(\rho\) & Shared-error correlation (groupthink strength) among agents within an aggregation group. & Run multiple agents on the same tasks and estimate pairwise correlation of signed correctness \(S_i=\widehat{Y}_iY\) (binary) or residuals \(E_i=X_i-Y\) (continuous). \\
\(W\) & Context window constraint controlling fan-in: each node can read at most \(W\) tokens. & System constraint minus templating/tool overhead; verify effective \(W\) in the deployed scaffold. \\
\(B\) & Total budget (tokens, time, calls) to be allocated across agents and communication. & Define a consistent additive accounting unit and record \(B\) from logs. \\
\hline
\end{tabular}
\caption{Key parameters of the framework and indicative estimation strategies. The theory treats these quantities as inputs; calibration is optional and necessarily coarse.}
\label{tab:measurable_params}
\end{table}

\begin{figure}[t]
\centering
\includegraphics[width=0.8\linewidth]{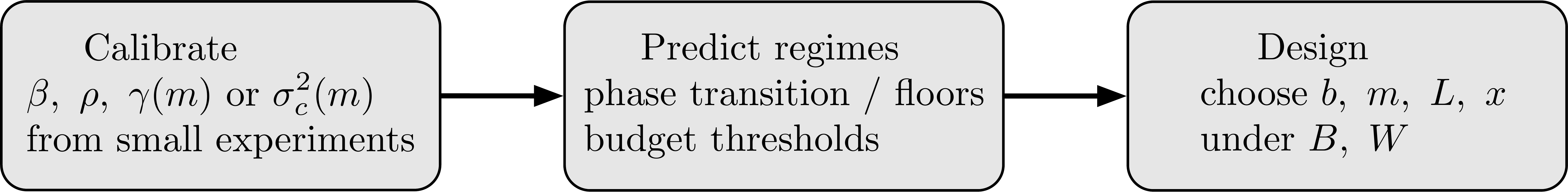}
\caption{A possible use of the framework. Estimate \(\beta\) (single-agent scaling), \(\rho\) (shared failures), and \(\gamma(m)\) or \(\sigma_c^2(m)\) (communication fidelity) in a given setting, then use the theory to map budgets \(B\) and context windows \(W\) to plausible organizations and message lengths.}
\label{fig:framework_loop}
\end{figure}

Sections~\ref{sec:continuous_warmup}--\ref{sec:design} use these parameters to derive correlation/communication floors, a binary amplification--collapse phase transition, and budget-dependent topology comparisons.

\section{Continuous Warm-up: Averaging Under Correlation and Communication Noise}
\label{sec:continuous_warmup}

We begin with a continuous-performance setting in which agents estimate a scalar target \(Y\in[0,1]\) and performance is measured by mean squared error (MSE).
This warm-up makes the basic organization trade-offs transparent in closed form: averaging reduces independent noise, shared errors impose floors, and communication loss introduces additional distortion.
Throughout we adopt the framework of Section~\ref{sec:framework} and use averaging at internal nodes, which is the natural choice under squared loss and symmetric noise assumptions.

Although we phrase the warm-up as each leaf producing an estimate of the same latent \(Y\), it should be read as an abstraction of role-specialized pipelines.
A child agent may solve a subtask and transmit an \(m\)-token message; \(X_i\) denotes the \emph{effective} scalar contribution about the downstream target that the parent can decode from that message.
Role heterogeneity then enters through the effective parameters (\(\rho\) and \(\sigma_c^2(m)\) or \(\gamma(m)\)); we return to richer semantics and extensions in Section~\ref{sec:discussion}.

\subsection{Model assumptions for the continuous setting}
\label{sec:cont_model}

This warm-up instantiates the continuous-task primitive from Section~\ref{sec:framework} using the communication and shared-failure models already defined in Sections~\ref{sec:communication} and \ref{sec:correlation}. 
The only additional assumptions are that leaf estimates are (approximately) unbiased and admit a compute-dependent MSE description.

Let \(Y\in[0,1]\) denote the latent target and let each leaf output an estimate \(X\in[0,1]\).
To keep the warm-up focused on organization rather than estimator bias, we assume conditional unbiasedness,
\begin{equation}
\mathbb{E}[X\mid Y]=Y.
\label{eq:cont_unbiased}
\end{equation}
Write the residual \(E:=X-Y\), so \(\mathbb{E}[E\mid Y]=0\), and summarize leaf quality by the MSE
\begin{equation}
v_0(x):=\mathbb{E}\big[(X-Y)^2\big]=\mathbb{E}[E^2],
\label{eq:v0_def}
\end{equation}
which depends on the per-leaf compute allocation \(x\).
In the operating range where we compare \textbf{scale up} versus \textbf{scale out}, we use the scaling approximation
\begin{equation}
v_0(x)\approx c\,x^{-\beta},
\label{eq:cont_scaling_recall}
\end{equation}
with constants \(c>0\) and \(\beta>0\) (Section~\ref{sec:scaling}).

Communication follows the additive distortion model from Section~\ref{sec:communication}, with message length \(m\); we treat edge distortions as independent across edges and independent of \(Y\).

Within any aggregation group, residuals are modeled by the equal-correlation structure from Section~\ref{sec:correlation}: for children \(i\neq j\),
\begin{equation}
\mathrm{Var}(E_i)=v,\qquad \mathrm{Cov}(E_i,E_j)=\rho v,
\label{eq:cont_corr}
\end{equation}
where \(\rho\in[0,1)\) is the shared-failure correlation parameter.

Finally, an internal node with fan-in \(b\) outputs the average of its received inputs,
\begin{equation}
\widehat{Y} = \frac{1}{b}\sum_{i=1}^b \widetilde{X}_i.
\label{eq:avg_rule}
\end{equation}
Under these symmetry assumptions, averaging is the MSE-optimal linear unbiased aggregator.

\subsection{Closed-form risk for star, chain, and tree}
\label{sec:cont_closed_form}

We now derive exact MSE expressions (or recursions) for the three canonical topologies.
These formulas already reveal three key messages: star saturates under context constraints, chains compound loss, and trees trade scale for multi-hop noise.

\paragraph{Star.}
Consider a star with \(N\) leaves feeding a single aggregator.
Each leaf sends a length-\(m\) message, so the aggregator is feasible only if \(Nm\le W\).
Let \(v_0\) denote the leaf MSE (for a given compute allocation \(x\)) and let \(\rho\) denote the pairwise residual correlation within this aggregation group.
After one hop, each received estimate has residual \(\widetilde{E}_i = (X_i-Y)+\eta_i\), with
\[
\mathrm{Var}(\widetilde{E}_i)=v_0 + \sigma_c^2(m),\qquad \mathrm{Cov}(\widetilde{E}_i,\widetilde{E}_j)=\rho v_0.
\]
The star output is the average of the received values, so the MSE is
\begin{equation}
v_{\text{star}}(N)
=
\mathbb{E}\big[(\widehat{Y}-Y)^2\big]
=
\frac{v_0\big(1+(N-1)\rho\big)}{N}
+
\frac{\sigma_c^2(m)}{N}.
\label{eq:star_mse}
\end{equation}

If \(N\) could grow arbitrarily, the channel noise term \(\sigma_c^2(m)/N\) would vanish.
However, correlation produces an irreducible floor:
\[
v_{\text{star}}(N)\ \to\ \rho v_0 \qquad \text{as } N\to\infty \text{ (for fixed } v_0\text{)}.
\]
Thus, even perfect communication cannot eliminate a globally shared residual component.
In practice, the context constraint \(Nm\le W\) prevents \(N\) from growing with budget once \(m\) is fixed, which is why star organizations often stop improving beyond a modest scale.

\paragraph{Chain.}
Consider a chain of length \(L\), where an estimate is handed off sequentially through \(L\) communication steps, each of length \(m\), without introducing new independent evidence.
Starting from a leaf estimate with MSE \(v_0\), each hop adds independent distortion variance \(\sigma_c^2(m)\).
Thus,
\begin{equation}
v_{\text{chain}}(L)=v_0 + L\,\sigma_c^2(m).
\label{eq:chain_mse}
\end{equation}

Equation~\eqref{eq:chain_mse} isolates the unavoidable cost of multi-hop \emph{relay} communication: if intermediate nodes primarily re-encode the same information (no new independent evidence), each hop injects distortion and MSE grows linearly in \(L\).
Pipeline-style chains can still be beneficial when each stage performs additional computation or incorporates new information; in that case, the estimation error may decrease across stages, but the same \(\sigma_c^2(m)\) term captures a communication bottleneck that accumulates with depth and must be outweighed by per-stage gains.

\paragraph{Hierarchical tree.}
Consider a full \(b\)-ary tree with depth \(L\), where each internal node averages its \(b\) children.
Let \(v_t\) denote the MSE of a node output at level \(t\) (leaves are \(t=0\); the root is \(t=L\)).
We assume that at each aggregation node, the \(b\) child residuals satisfy the equal-correlation model~\eqref{eq:cont_corr} with coefficient \(\rho\), and that channel noise on incoming edges is independent.

A parent receives \(\widetilde{X}_i = X_i + \eta_i\) from each child, so each received residual has variance \(v_t+\sigma_c^2(m)\), while covariances remain \(\rho v_t\) because the channel noises are independent.
Averaging~\eqref{eq:avg_rule} yields the recursion
\begin{equation}
v_{t+1}
=
a\,v_t
+\frac{\sigma_c^2(m)}{b},
\qquad
a:=\frac{1+(b-1)\rho}{b}=\rho + \frac{1-\rho}{b}.
\label{eq:tree_rec}
\end{equation}
Unrolling gives the closed form
\begin{equation}
v_L
=
a^L v_0
+
\frac{\sigma_c^2(m)}{b}\cdot \frac{1-a^L}{1-a}.
\label{eq:tree_closed}
\end{equation}

The coefficient \(a\in(0,1)\) when \(\rho<1\), so the first term \(a^L v_0\) shrinks exponentially in depth.
This is the statistical benefit of hierarchy: averaging repeatedly reduces the variance of what is being passed upward.
The second term is the cost: each level injects fresh communication distortion, which accumulates into an error floor described next.

\subsection{Saturation floors and mixing depth}
\label{sec:cont_floors}

Equation~\eqref{eq:tree_closed} exposes a sharp and practical phenomenon: deeper trees do not improve indefinitely.
Even if leaves become arbitrarily accurate, repeated communication loss produces a floor.
Figure~\ref{fig:cont_saturation} illustrates the typical behavior: $v_L$ shrinks rapidly with depth at first (roughly at rate $a^L$) and then saturates at the floor $v^\star$ determined by $\sigma_c^2(m)$, $b$, and $\rho$.

\paragraph{Tree floor.}
For \(\rho<1\) we have \(a<1\), and therefore \(a^L\to 0\) as \(L\to\infty\).
Taking limits in~\eqref{eq:tree_closed} yields
\begin{equation}
v^\star(b,m,\rho)
:=
\lim_{L\to\infty} v_L
=
\frac{\sigma_c^2(m)}{b(1-a)}
=
\frac{\sigma_c^2(m)}{(b-1)(1-\rho)}.
\label{eq:tree_floor}
\end{equation}

This expression cleanly separates the levers available to a designer:
\begin{itemize}
  \item Increasing message length \(m\) decreases \(\sigma_c^2(m)\), lowering the floor.
  \item Increasing fan-in \(b\) lowers the floor approximately as \(1/(b-1)\), but \(b\) is limited by \(bm\le W\).
  \item Reducing shared correlation \(\rho\) lowers the floor and also accelerates convergence (since \(a\) decreases).
\end{itemize}

\paragraph{Mixing depth.}
The recursion~\eqref{eq:tree_rec} is linear, so convergence to the floor is explicit:
\begin{equation}
v_L - v^\star
=
a^L\,(v_0 - v^\star).
\label{eq:tree_mix_identity}
\end{equation}
Thus, to achieve a relative tolerance \(v_L \le (1+\varepsilon)v^\star\) when \(v_0>v^\star\), it suffices to choose
\begin{equation}
L \ \ge\
\left\lceil
\frac{\log\big((v_0-v^\star)/(\varepsilon v^\star)\big)}{\log(1/a)}
\right\rceil.
\label{eq:cont_mix_depth}
\end{equation}
Equation~\eqref{eq:cont_mix_depth} is the continuous analogue of the ``mixing depth'' results we prove for binary trees later.
It is also a direct design rule: beyond a certain depth, additional leaves yield negligible improvement unless the designer reduces the floor by improving communication or reducing correlation.

\paragraph{A note on when deeper is worse.}
If leaves are already more accurate than the floor (i.e., \(v_0 < v^\star\)), then \eqref{eq:tree_closed} implies that increasing depth moves error \emph{upward} toward \(v^\star\).
In that regime, deeper hierarchies are counterproductive unless they come with reduced \(\sigma_c^2(m)\) or reduced \(\rho\).
This observation will reappear in the binary setting as a saturation effect: once the system has reached its fixed point, additional scale should be spent on changing the communication or correlation regime rather than adding depth.

\begin{figure}[t]
  \centering
  \includegraphics[width=0.7\linewidth]{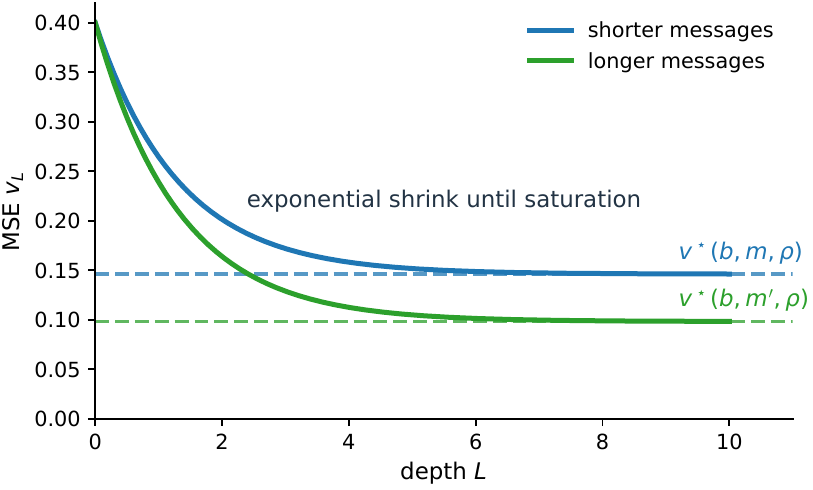}
  \caption{
  Continuous averaging recursion through a $b$-ary hierarchy with equal-correlation $\rho$.
  The MSE $v_L$ contracts approximately exponentially with depth $L$ until it reaches the
  communication-limited floor
  $v^\star(b,m,\rho)=\sigma_c^2(m)/\bigl((b-1)(1-\rho)\bigr)$.
  Longer messages reduce $\sigma_c^2(m)$ and therefore lower the attainable floor.
  }
  \label{fig:cont_saturation}
\end{figure}

\subsection{Scale-out versus scale-up under a fixed budget}
\label{sec:cont_budget_tradeoff}

The closed forms above become design tools once we connect them to budget.
We use the budget model from Section~\ref{sec:budget-context}, where each leaf receives compute \(x\) and each edge uses message length \(m\).
For a full \(b\)-ary tree of depth \(L\) with \(N=b^L\) leaves, the total cost is approximately
\[
B \;\approx\; N\big(x + c_0(b,m)\big),
\qquad
c_0(b,m)=2m\frac{b}{b-1},
\]
so the number of leaves is \(N\approx B/(x+c_0)\), and depth is \(L\approx \log_b N\).

\paragraph{A growth-regime approximation.}
When the tree has not yet hit the communication floor, the dominant term in \eqref{eq:tree_closed} is the shrinking term \(a^L v_0(x)\).
Using \(N=b^L\) and \(a^L = N^{-t}\) with
\begin{equation}
t(b,\rho) := -\frac{\log a}{\log b}
= -\frac{\log\big(\rho+(1-\rho)/b\big)}{\log b},
\label{eq:cont_org_exponent}
\end{equation}
we obtain the approximation
\begin{equation}
v_L \;\approx\; v_0(x)\,N^{-t(b,\rho)}.
\label{eq:cont_growth}
\end{equation}
The exponent \(t\in[0,1]\) quantifies how efficiently scaling out reduces error in the continuous setting.
If \(\rho=0\), then \(a=1/b\) and \(t=1\), recovering the familiar \(1/N\) averaging benefit.
As \(\rho\) increases, \(a\) approaches 1 and \(t\) approaches 0, reflecting that shared residuals diminish the value of adding more agents.

\paragraph{Compute allocation and a simple threshold.}
Substituting \(v_0(x)\approx c x^{-\beta}\) and \(N\approx B/(x+c_0)\) into \eqref{eq:cont_growth} yields
\begin{equation}
v_L \;\approx\; c\,x^{-\beta}\left(\frac{B}{x+c_0}\right)^{-t}
= c\,B^{-t}\,x^{-\beta}(x+c_0)^t.
\label{eq:vL-2}
\end{equation}
For a fixed budget \(B\), minimizing this expression over \(x\) reduces to minimizing \(x^{-\beta}(x+c_0)^t\).
A short calculus argument shows a qualitative threshold:
\begin{itemize}
  \item If \(t>\beta\), there is a unique interior optimum
  \begin{equation}
  x^\star \;=\; \frac{\beta}{t-\beta}\,c_0(b,m).
  \label{eq:cont_xstar}
  \end{equation}
  In this regime, scale-out and scale-up should be balanced: make agents strong enough that their improvement exponent \(\beta\) is not wasted, but keep them weak enough that adding more of them leverages the organization exponent \(t\).
  \item If \(t\le \beta\), the objective decreases as \(x\) increases, suggesting that under the growth approximation it is better to spend budget on fewer stronger agents rather than more weaker ones.
  Correlation and context constraints tend to push systems into this regime by reducing \(t\) (via higher \(\rho\)) or restricting \(b\), which increases communication overhead \(c_0(b,m)\) and raises the floor \(v^\star(b,m,\rho)\), shortening the growth regime.
\end{itemize}
This is the continuous analogue of our main binary condition \(s>\beta\) developed later: synergy from scaling out appears only when the organization exponent exceeds the single-agent scaling exponent.

\paragraph{How the floor changes the budget story.}
The growth approximation \eqref{eq:cont_growth} cannot hold indefinitely because of the communication floor \(v^\star\).
Once the predicted \(v_L\) approaches \(v^\star(b,m,\rho)\), additional budget spent on more leaves or more depth yields diminishing returns.
At that point, the effective levers are no longer \(N\) and \(L\), but \(m\) and diversity (which affect \(\sigma_c^2(m)\) and \(\rho\)).
This conclusion will reappear in the binary setting through fixed-point saturation and mixing depth.

\subsection{Design implications and a bridge to the binary analysis}
\label{sec:cont_implications}

The continuous warm-up already exposes the structural trade-offs we will reuse in the binary theory.
Absent a binding context constraint, star aggregation is information-dominant because it avoids multi-hop loss; trees are useful primarily as a way to bypass the fan-in limit \(Nm\le W\) by enforcing \(bm\le W\) locally, at the cost of repeated communication.
Both correlation and communication loss induce explicit performance floors, so scale-out can stall early when \(\rho\) is large or messages are too short.
In the growth regime under a fixed budget, the key comparison is exponent versus exponent: whether the organization exponent beats the single-agent scaling exponent.

The binary case is qualitatively different.
Because majority vote is nonlinear, it can fundamentally change how errors propagate under repeated aggregation.
Section~\ref{sec:binary_core} analyzes this effect and shows that hierarchical organization exhibits a sharp amplification--collapse transition on success/failure tasks.

\section{Binary Core: Majority Aggregation and Phase Transitions}
\label{sec:binary_core}
We begin by characterizing how one majority layer transforms weak bias, then incorporate correlated errors and lossy communication to obtain a phase transition and budget-relevant exponents for deep trees.

\subsection{Majority maps and correlated aggregation}
\label{sec:maj_maps}

The central object in the binary analysis is the one-step aggregation map that describes how an internal node transforms the biases of its child votes into an output bias.

\paragraph{Majority aggregation map.}
Let \(b\ge 3\) be an odd integer and let \(V_1,\dots,V_b\in\{-1,+1\}\) be i.i.d. child votes with
\[
\mathbb{E}[V_i Y] = u \in [-1,1].
\]
Equivalently, \(V_i\) is correct with probability \(p=(1+u)/2\).
Define the majority vote
\[
\mathrm{Maj}(V_1,\dots,V_b) := \mathrm{sign}\!\Big(\sum_{i=1}^b V_i\Big)\in\{-1,+1\},
\]
where ties are impossible because \(b\) is odd; even $b$ can be handled by specifying a tie-breaking rule (e.g., random) with only minor notational changes. For the same reason, we will assume an odd $b$ in our future sections.
The induced \emph{majority map} is the bias at the output:
\begin{equation}
f_b(u)
\;:=\;
\mathbb{E}\big[\mathrm{Maj}(V_1,\dots,V_b)\,Y\big]
=
2\,\Pr\!\Big(\mathrm{Bin}(b,p)\ge \tfrac{b+1}{2}\Big)-1,
\qquad p=\frac{1+u}{2}.
\label{eq:fb_def}
\end{equation}
The function \(f_b\) summarizes the effect of one noiseless majority aggregation step on the bias.

\begin{lemma}[Basic properties of the majority map]
\label{lem:fb_props}
For odd \(b\ge 3\), the map \(f_b:[-1,1]\to[-1,1]\) satisfies:
\begin{enumerate}
  \item \(f_b\) is odd and increasing, with \(f_b(0)=0\) and \(f_b(1)=1\).
  \item \(f_b\) is concave on \([0,1]\).
  \item The derivative at the origin is
  \begin{equation}
  f_b'(0)=\frac{b\binom{b-1}{(b-1)/2}}{2^{b-1}}.
  \label{eq:fb_prime0}
  \end{equation}
\end{enumerate}
\end{lemma}

Lemma~\ref{lem:fb_props} is standard for majority maps; we provide a self-contained proof in Appendix~A.
The quantity \(f_b'(0)\) is especially important.
It is the linear gain of the majority when the inputs are only slightly better than chance.
As \(b\) increases, \(f_b'(0)\) grows sublinearly (in fact on the order of \(\sqrt{b}\)), so increasing fan-in yields diminishing marginal amplification. Figure~\ref{fig:majority_map} visualizes $f_b(u)$ for several fan-ins, highlighting its concavity and the increasing (but sublinear) small-signal gain $f_b'(0)$.

\begin{figure}[!ht]
  \centering
  \includegraphics[width=0.70\linewidth]{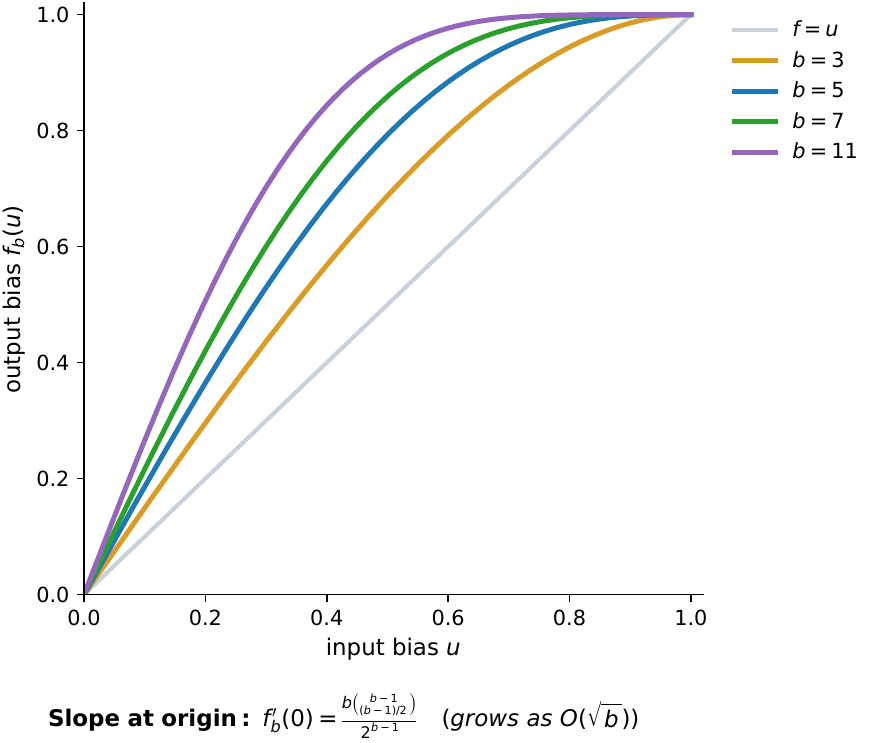}
  \caption{
  Majority amplification map $f_b(u)$ for odd fan-in $b$ (shown for $b\in\{3,5,7,11\}$),
  where $u$ is the input bias and $f_b(u)=2\Pr(\mathrm{Bin}(b,(1+u)/2)\ge (b+1)/2)-1$.
  The diagonal $f(u)=u$ is shown for reference. The slope at the origin,
  $f_b'(0)=\frac{b\binom{b-1}{(b-1)/2}}{2^{\,b-1}}$, grows on the order of $\sqrt{b}$,
  capturing how majority increasingly amplifies weak signals as fan-in grows.
  }
  \label{fig:majority_map}
\end{figure}

\paragraph{A tractable shared-failure model.}
To capture groupthink, we need a model in which child votes are positively correlated.
We use the following local generative model, which matches a measurable correlation coefficient and yields a closed-form aggregation rule.

\begin{definition}[\(\rho\)-shared correlation model]
\label{def:rho_shared}
Fix a target bias \(u\in[-1,1]\) and correlation strength \(\rho\in[0,1)\).
Conditional on \(Y\), the child votes \(V_1,\dots,V_b\in\{-1,+1\}\) are generated as follows:
with probability \(\rho\), all children share a common vote \(V_1=\cdots=V_b=:V\) where \(\mathbb{E}[VY]=u\);
with probability \(1-\rho\), the votes are i.i.d. with \(\mathbb{E}[V_iY]=u\).
\end{definition}

This model has two practical advantages.
First, \(\rho\) equals the pairwise correlation of signed correctness.
Let \(S_i:=V_iY\in\{-1,+1\}\).
Then \(\mathrm{Corr}(S_i,S_j)=\rho\) for any \(i\neq j\), so \(\rho\) can be estimated from logs as described in Section~\ref{sec:correlation}.
Second, the induced aggregation map is explicit.
Figure~\ref{fig:rho_model} provides an intuition for the mixture: with probability $\rho$ the children act as a single shared voter (groupthink), and with probability $1-\rho$ they behave as independent voters.

\begin{lemma}[Correlated majority map under the \(\rho\)-shared model]
\label{lem:correlated_map}
Under Definition~\ref{def:rho_shared}, the output bias of majority aggregation is
\begin{equation}
f_{b,\rho}(u)
\;:=\;
\mathbb{E}\big[\mathrm{Maj}(V_1,\dots,V_b)\,Y\big]
=
\rho\,u + (1-\rho)\,f_b(u).
\label{eq:fbrho_def}
\end{equation}
In particular, \(f_{b,\rho}\) is increasing and concave on \([0,1]\), and
\begin{equation}
f_{b,\rho}'(0)=\rho+(1-\rho)f_b'(0).
\label{eq:fbrho_prime0}
\end{equation}
\end{lemma}

Lemma~\ref{lem:correlated_map} is immediate from the mixture construction: in the shared mode, majority returns the same vote (bias \(u\)); in the independent mode, it returns \(f_b(u)\).
The key design message is already visible in \eqref{eq:fbrho_prime0}:
increasing \(\rho\) moves \(f_{b,\rho}\) closer to the identity map, weakening amplification and making deep organization harder.

\begin{figure}[t]
\centering
  \includegraphics[width=0.75\linewidth]{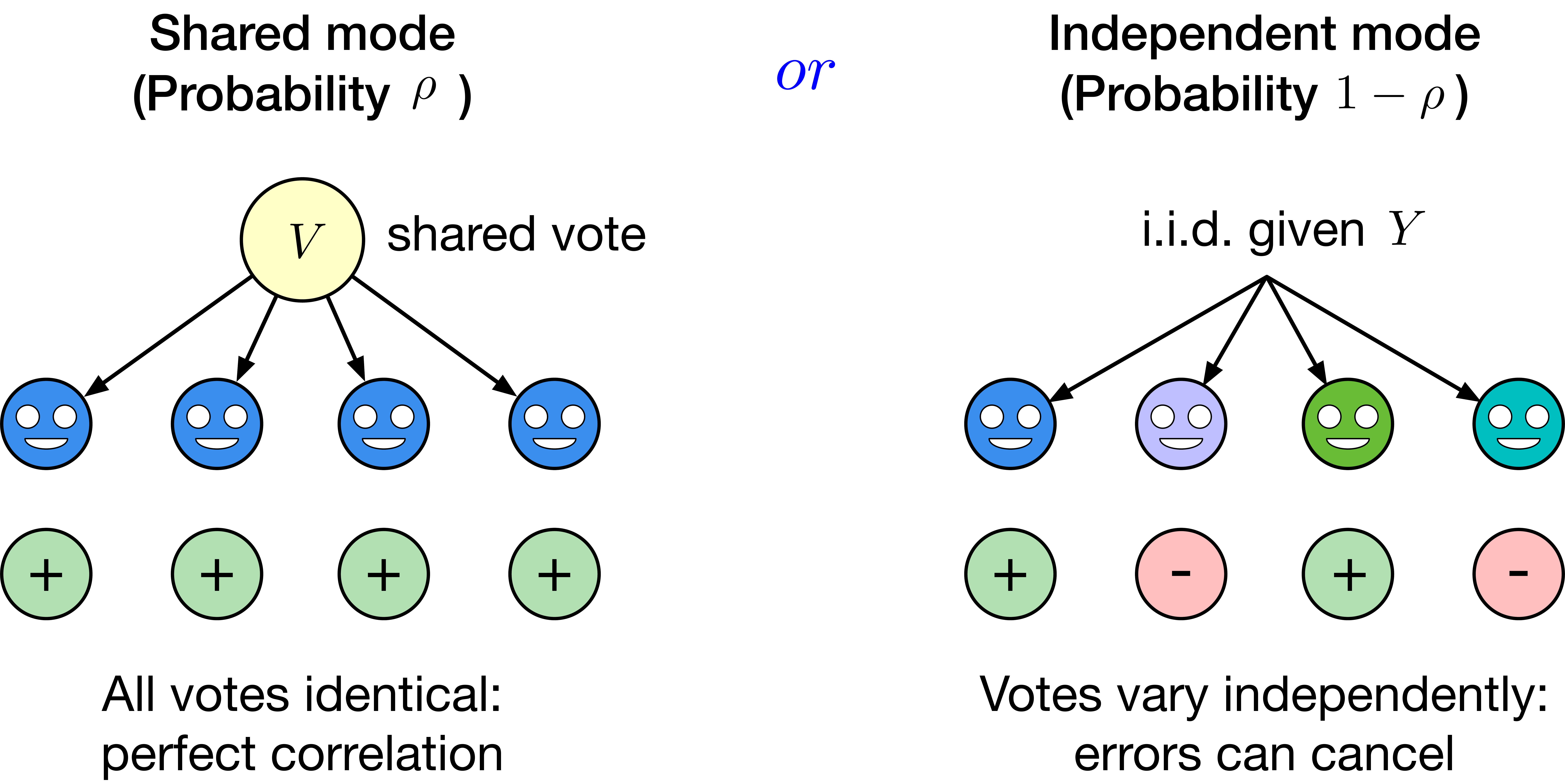}

\caption{The $\rho$-shared correlation model (Definition~\ref{def:rho_shared}).
With probability $\rho$, all $b$ child agents share a common vote $V$ (left), producing perfect within-group correlation. This captures ``groupthink'' where agents fail together.
With probability $1-\rho$, agents vote independently given the truth $Y$ (right), allowing errors to cancel through aggregation.
The parameter $\rho$ equals the pairwise correlation of signed correctness and can be estimated from logs.}
\label{fig:rho_model}
\end{figure}

\subsection{Phase transition in deep trees: \texorpdfstring{$\alpha_\rho \gtrless 1$}{alpha rho > 1}}
\label{sec:phase_transition}

We now incorporate communication and study deep trees.
Consider a homogeneous \(b\)-ary tree in which every internal node aggregates \(b\) child votes by majority, and every edge transmits a message of length \(m\) tokens.
Communication is modeled as an effective bit channel with reliability \(\gamma(m)\in(0,1]\) (Section~\ref{sec:communication}):
if a child vote has bias \(\mu\), then the parent receives a vote with bias \(\gamma(m)\mu\).

\paragraph{One-layer bias recursion.}
Let \(\mu_t\) denote the bias of a node at depth \(t\) from the leaves (so leaves are \(t=0\)).
After one hop, each child vote's marginal bias is attenuated to \(\gamma(m)\mu_t\).
We model the \(b\) received votes entering each majority gate by the \(\rho\)-shared model (Definition~\ref{def:rho_shared}) with this marginal bias and within-group correlation \(\rho\). (Any additional dependence induced by shared message templates or decoding is absorbed into the effective \(\rho\).)
Under this abstraction, the bias evolves as
\begin{equation}
\mu_{t+1} = T(\mu_t),
\qquad
T(u) := f_{b,\rho}\!\big(\gamma(m)\,u\big).
\label{eq:tree_recursion_binary}
\end{equation}
The map \(T\) is increasing, concave on \([0,1]\), and satisfies \(T(0)=0\).

\paragraph{The effective gain \(\alpha_\rho\).}
The local behavior of \(T\) near the origin is governed by the derivative
\begin{equation}
\alpha_\rho
:= T'(0)
=
\gamma(m)\,f_{b,\rho}'(0)
=
\gamma(m)\big(\rho+(1-\rho)f_b'(0)\big).
\label{eq:alpha_rho_def}
\end{equation}
We interpret \(\alpha_\rho\) as the \emph{effective per-layer gain} on weak signal.
If leaf votes are only slightly better than random, then one layer of aggregation and communication multiplies their bias by approximately \(\alpha_\rho\).
This single scalar already suggests a threshold: if \(\alpha_\rho<1\), weak signal shrinks from layer to layer; if \(\alpha_\rho>1\), it grows.

The next theorem shows that this intuition is exact at the level of global dynamics: \(\alpha_\rho\) determines whether deep trees amplify accuracy or collapse to chance.

\begin{theorem}[Phase transition for deep majority trees]
\label{thm:phase_transition}
Fix odd \(b\ge 3\), correlation \(\rho\in[0,1)\), and channel reliability \(\gamma=\gamma(m)\in(0,1]\).
Let \(T(u)=f_{b,\rho}(\gamma u)\) and \(\alpha_\rho=T'(0)\) as in \eqref{eq:tree_recursion_binary}--\eqref{eq:alpha_rho_def}.
Consider the recursion \(\mu_{t+1}=T(\mu_t)\) with \(\mu_0\in[0,1]\).
\begin{enumerate}
  \item (\emph{Subcritical collapse.}) If \(\alpha_\rho \le 1\), then \(\mu_t \to 0\) as \(t\to\infty\).
  \item (\emph{Supercritical amplification and saturation.}) If \(\alpha_\rho > 1\), then there exists a unique fixed point \(\mu^\star\in(0,1]\) such that \(T(\mu^\star)=\mu^\star\), and \(\mu_t\to \mu^\star\) for every \(\mu_0\in(0,1]\).
  Moreover, \(\mu^\star=1\) if \(\gamma=1\), and \(\mu^\star<1\) if \(\gamma<1\).
\end{enumerate}
\end{theorem}

Figure~\ref{fig:binary_phase_transition} illustrates the two regimes: when $\alpha_\rho<1$ the recursion map lies below the diagonal and $\mu_t$ collapses to $0$, while when $\alpha_\rho>1$ the map crosses the diagonal at a stable fixed point $\mu^\star$, yielding amplification followed by saturation.

\paragraph{Proof sketch.}
Because \(T\) is concave with \(T(0)=0\), the ratio \(r(u):=T(u)/u\) is nonincreasing on \(u\in(0,1]\).
Its limit at the origin is \(r(0^+)=T'(0)=\alpha_\rho\).
If \(\alpha_\rho\le 1\), then \(T(u)\le u\) for all \(u\in(0,1]\), so the recursion is monotone decreasing and must converge to a fixed point; the only possibility is 0.
If \(\alpha_\rho>1\), then \(r(0^+)>1\), while \(r(1)=T(1)\le 1\) with strict inequality when \(\gamma<1\).
By monotonicity of \(r\), there is a unique \(u\) where \(r(u)=1\), yielding a unique \(\mu^\star\).
The sign of \(T(u)-u\) implies the recursion moves toward \(\mu^\star\) from any starting point.
Full details are given in Appendix~C. \hfill\(\square\)

\paragraph{Connection to reconstruction thresholds.}
In the classical broadcasting problem on trees with a binary symmetric channel (BSC) of second eigenvalue \(\theta\), the Kesten--Stigum condition \(b\theta^2>1\) characterizes when the root remains reconstructable from the leaves (it is tight for the two-state symmetric model and, more generally, gives the Kesten--Stigum bound for Bayes-optimal reconstruction) \citep{kesten1966additional,evans2000broadcasting,mossel2001reconstruction}.
Our condition \(\alpha_\rho>1\) plays an analogous role for the specific \emph{majority-based} recursion under our correlated-input model and effective reliability \(\gamma(m)\): it is exactly the linear gain \(T'(0)\) of the update map at the origin.
We emphasize that \(\alpha_\rho\) is estimator-dependent (majority rather than Bayes-optimal) and explicitly incorporates shared failures through \(\rho\).

\paragraph{What the threshold means in practice.}
Theorem~\ref{thm:phase_transition} clarifies why adding hierarchical structure is risky.
In a supercritical regime, a tree can turn many weak votes into a strong decision and is robust to depth until it saturates.
In a subcritical regime, depth destroys information: the system becomes \emph{less} reliable as it grows.
Both correlation (\(\rho\)) and short messages (small \(\gamma(m)\)) push \(\alpha_\rho\) downward, making collapse more likely.
This matches the practitioner experience that groupthink and aggressive compression are the two most common reasons multi-agent scaling fails.

\begin{figure}[t]
  \centering
  \begin{subfigure}[t]{0.485\linewidth}
    \centering
    \includegraphics[width=\linewidth]{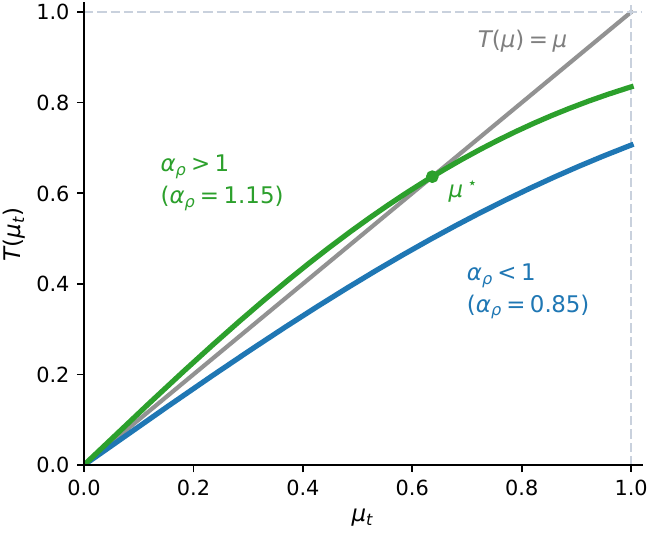}
    \caption{Recursion map $T(\mu)$ vs.\ the diagonal.}
  \end{subfigure}
  \hfill
  \begin{subfigure}[t]{0.485\linewidth}
    \centering
    \includegraphics[width=\linewidth]{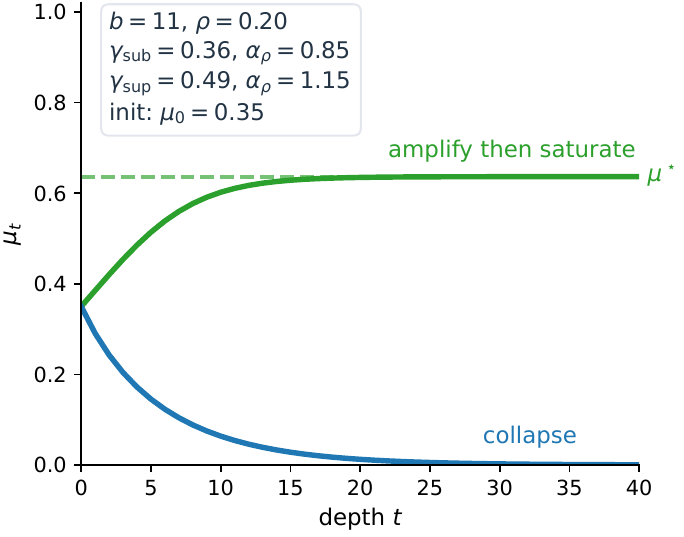}
    \caption{Iterates $\mu_{t+1}=T(\mu_t)$ from the same initialization.}
  \end{subfigure}
  \caption{
  Binary phase transition under the recursion map $T(\mu)=f_{b,\rho}(\gamma\mu)$ where
  $f_{b,\rho}(u)=\rho u+(1-\rho)f_b(u)$.
  In the subcritical regime ($\alpha_\rho<1$) trajectories collapse to $\mu=0$.
  In the supercritical regime ($\alpha_\rho>1$) a stable fixed point $\mu^\star>0$ emerges
  and trajectories amplify then saturate at $\mu^\star$.
  }
  \label{fig:binary_phase_transition}
\end{figure}

\subsection{Small-signal amplification and the organization exponent \texorpdfstring{$s$}{s}}
\label{sec:small_signal}

The phase transition determines whether deep organization is viable at all.
When it is viable (\(\alpha_\rho>1\)), the next question is \emph{how fast} accuracy can grow with the number of leaves.
This is where a simple exponent emerges.

\paragraph{Linear regime and a global upper bound.}
Concavity of \(T\) implies a strong inequality: the tangent line at the origin upper-bounds the entire map on \([0,1]\),
\begin{equation}
T(u) \le T'(0)\,u = \alpha_\rho\,u,\qquad u\in[0,1].
\label{eq:global_tangent_bound}
\end{equation}
Thus, for a depth-\(L\) tree,
\begin{equation}
\mu_L \le \mu_0\,\alpha_\rho^{L}.
\label{eq:mu_upper_alphaL}
\end{equation}
This bound is informative even when \(\alpha_\rho>1\): it tells us that exponential growth in depth is the fastest possible behavior in the small-signal regime.

To complement \eqref{eq:mu_upper_alphaL} we also need a lower bound showing that the upper bound is tight when signals remain small.
Because \(T\) is differentiable at 0, its local behavior is well-approximated by its derivative.

\begin{theorem}[Small-signal amplification band]
\label{thm:small_signal_band}
Fix odd \(b\ge 3\), \(\rho\in[0,1)\), and channel reliability \(\gamma=\gamma(m)\in(0,1]\). Let \(T(u)=f_{b,\rho}(\gamma u)\) and \(\alpha_\rho:=T'(0)\).
For any \(\eta\in(0,1)\), there exists \(\delta=\delta(\eta;b,\rho,\gamma)\in(0,1]\) such that for all \(u\in[0,\delta]\),
\begin{equation}
(1-\eta)\,\alpha_\rho\,u \;\le\; T(u) \;\le\; \alpha_\rho\,u.
\label{eq:linear_band}
\end{equation}
Consequently, if \(\mu_0\,\alpha_\rho^{L}\le \delta\), then the depth-\(L\) recursion satisfies
\begin{equation}
(1-\eta)^{L}\,\mu_0\,\alpha_\rho^{L}
\;\le\;
\mu_L
\;\le\;
\mu_0\,\alpha_\rho^{L}.
\label{eq:mu_sandwich}
\end{equation}
\end{theorem}

\paragraph{Proof sketch.}
The upper bound is \eqref{eq:global_tangent_bound}.
For the lower bound, differentiability of \(T\) at 0 implies \(T(u)/u \to \alpha_\rho\) as \(u\downarrow 0\).
Thus there exists \(\delta\) such that \(T(u)\ge (1-\eta)\alpha_\rho u\) on \([0,\delta]\).
If \(\mu_0\alpha_\rho^L\le\delta\), then the upper bound ensures \(\mu_t\le\delta\) for all \(t\le L\), so the lower inequality applies at every step and yields \eqref{eq:mu_sandwich}.
Full details appear in Appendix~D. \hfill\(\square\)

\paragraph{The organization exponent.}
For a full \(b\)-ary tree of depth \(L\), the number of leaves is \(N=b^L\).
Theorem~\ref{thm:small_signal_band} implies that in the growth regime (before saturation),
\begin{equation}
\mu_L \approx \mu_0\,\alpha_\rho^{L}
= \mu_0\,N^{s},
\qquad
s := \frac{\log \alpha_\rho}{\log b}.
\label{eq:organization_exponent}
\end{equation}
We call \(s\) the \emph{organization exponent}.
It converts the per-layer gain \(\alpha_\rho\) into a statement about how performance scales with the number of leaves under hierarchical aggregation.

This exponent is the main interface between organization and budget.
Later, in Section~\ref{sec:phase_diagrams}, we compare \(s\) to the single-agent scaling exponent \(\beta\) to decide whether scale-out can beat scale-up under a fixed budget.
At this stage, the key point is conceptual: \(s\) increases with better communication (\(\gamma(m)\)), decreases with shared failures (\(\rho\)), and is limited by fan-in through \(f_b'(0)\).

\subsection{Saturation and finite-depth guarantees: \texorpdfstring{$L_{\mathrm{mix}}$}{Lmix} and clipped objectives}
\label{sec:saturation_binary}

Small-signal amplification does not continue forever.
Bias is bounded by 1, and communication loss and correlation prevent arbitrarily deep trees from achieving perfect accuracy unless the channel is lossless.
Theorem~\ref{thm:phase_transition} already tells us that in the supercritical regime the recursion converges to a fixed point \(\mu^\star\).
In practice, we need two additional pieces of information:
how quickly \(\mu_t\) approaches \(\mu^\star\), and how to build a design objective that remains accurate across both growth and saturation regimes.

\paragraph{Mixing depth.}
Because \(T\) is concave and \(T(\mu^\star)=\mu^\star\), the slope at the fixed point satisfies \(T'(\mu^\star)\le 1\), with strict inequality whenever \(\gamma<1\) or \(\rho<1\).
This creates a locally contracting region around \(\mu^\star\), which implies geometric convergence once the recursion enters that region.

\begin{theorem}[Finite-depth convergence to the fixed point]
\label{thm:mixing_depth}
Assume \(\alpha_\rho>1\) and \(\rho\in[0,1)\).
Let \(\mu^\star\in(0,1]\) be the unique fixed point of \(T\) from Theorem~\ref{thm:phase_transition}.
For any \(\varepsilon\in(0,1)\), there exists a finite depth \(L_{\mathrm{mix}}(\varepsilon)\) such that
\[
\mu_L \ge (1-\varepsilon)\mu^\star
\qquad \text{for all } L \ge L_{\mathrm{mix}}(\varepsilon)
\]
whenever \(\mu_0>0\).
Moreover, \(L_{\mathrm{mix}}(\varepsilon)\) can be chosen to scale as
\[
L_{\mathrm{mix}}(\varepsilon)
=
O\!\left(\log\frac{1}{\mu_0} + \log\frac{1}{\varepsilon}\right),
\]
with constants depending only on \((b,\rho,\gamma)\).
\end{theorem}

\paragraph{Proof sketch.}
Write the recursion as \(\mu_{t+1}=T(\mu_t)\) with \(T(u)=f_{b,\rho}(\gamma u)\) and a stable fixed point \(\mu^\star\) when \(\alpha_\rho=T'(0)>1\).
There are two phases.
First, while \(\mu_t\) remains in the small-signal band, Theorem~\ref{thm:small_signal_band} gives \((1-\eta)\alpha_\rho^t\mu_0 \le \mu_t \le \alpha_\rho^t\mu_0\), so reaching a constant-sized neighborhood of \(\mu^\star\) takes \(O(\log(1/\mu_0))\) steps.
Second, stability implies \(T'(\mu^\star)<1\); by continuity there exists a neighborhood below \(\mu^\star\) on which \(T'(\mu)\le \kappa<1\), so the recursion is a contraction and the error shrinks geometrically, reaching relative error \(\varepsilon\) in \(O(\log(1/\varepsilon))\) additional steps.
Appendix~E provides an explicit construction and constants. \hfill\(\square\)

\paragraph{A clipped objective for design.}
The results above suggest a simple approximation that is both interpretable and safe:
use the small-signal prediction \(\mu_0\alpha_\rho^L\) when it is below the saturation point, and clip it at \(\mu^\star\) otherwise:
\begin{equation}
\widehat{\mu}_L
:=
\min\big\{\mu^\star,\ \mu_0\,\alpha_\rho^{L}\big\}.
\label{eq:clipped_mu}
\end{equation}
The clipped objective captures the two dominant regimes with a single expression.
In the common small-signal regime \(\mu_0\le\mu^\star\), concavity yields the global upper bound \(\mu_L\le \widehat{\mu}_L\), while Theorem~\ref{thm:small_signal_band} guarantees that \(\mu_L\) tracks the growth term when \(\widehat{\mu}_L=\mu_0\alpha_\rho^L\).
Theorem~\ref{thm:mixing_depth} guarantees that \(\mu_L\) tracks the saturation term when \(\widehat{\mu}_L=\mu^\star\) and depth exceeds \(L_{\mathrm{mix}}\).

\paragraph{Design takeaway.}
Equation~\eqref{eq:clipped_mu} makes a practical point precise.
In the supercritical regime, adding depth and leaves helps only until the system approaches \(\mu^\star\).
Beyond that point, additional budget should be spent on changing \(\mu^\star\) itself, which requires improving communication fidelity (increasing \(\gamma(m)\) by using longer messages) or reducing shared-error correlation \(\rho\) (increasing diversity or adding verification).
In the subcritical regime, adding depth is actively harmful; the only viable route to scale-out is to move the system across the threshold \(\alpha_\rho>1\) by improving \(m\), reducing \(\rho\), or increasing fan-in \(b\) within the context constraint.

Section~\ref{sec:phase_diagrams} uses the organization exponent \(s\), the threshold \(\alpha_\rho>1\), and the clipped objective \eqref{eq:clipped_mu} to derive topology and budget phase diagrams under context limits.
Section~\ref{sec:design} turns these results into a theory-guided design algorithm that outputs monotone communication design curves and compute allocations under a fixed budget.

\subsection{Topology and Budget Phase Diagrams}
\label{sec:phase_diagrams}

Building on the binary recursion and organization exponent from Section~\ref{sec:binary_core}, we translate the layer-wise results into system-level guidance under a fixed budget \(B\) and context window \(W\).
We ask when a context-limited star is feasible, when hierarchy is necessary, and when scaling out can beat scaling up before saturation.
The resulting regimes can be summarized as phase diagrams in the measurable environment parameters \(\beta\), \(\rho\), and \(\gamma(m)\) (Section~\ref{sec:framework}), together with the context constraint \(W\).

\subsubsection{Star as an upper bound without constraints}
\label{sec:star_upper_bound}

A recurring empirical pattern is that hierarchies sometimes outperform naive centralization.
It is tempting to interpret this as ``hierarchy is intrinsically better.''
Without constraints, that interpretation is false.
A star is not merely a particular topology; it is the topology that gives the decision rule access to \emph{all} leaf information in a single place.
Any hierarchy that compresses information on the way up can only discard information, not create it \citep{cover2006elements}.

We formalize this point in an intentionally general way, independent of the details of the aggregation rule.

\begin{proposition}[Centralization dominates under unlimited context]
\label{prop:centralization}
\leavevmode\\
Let \(X=(X_1,\dots,X_N)\) denote the collection of all leaf messages (or leaf outputs) generated in response to a task with truth label \(Y\).
Consider any multi-agent protocol on any directed acyclic topology such that the root ultimately observes a variable \(Z\) that is a (possibly randomized) function of \(X\) with randomness independent of \(Y\).
Then for any loss \(\ell(\widehat{Y},Y)\),
\[
\inf_{\widehat{Y}=\phi(Z)} \mathbb{E}[\ell(\widehat{Y},Y)]
\;\ge\;
\inf_{\widehat{Y}=\psi(X)} \mathbb{E}[\ell(\widehat{Y},Y)].
\]
In particular, a centralized decision rule $\psi(X)$ that has access to \(X\) cannot perform worse than any rule $\phi(Z)$ that only sees \(Z\).
\end{proposition}


\paragraph{Proof.}
Let $U$ collect the protocol's internal randomness, independent of $Y$, so that $Z=g(X,U)$.
Any rule $\phi(Z)$ induces a (possibly randomized) rule $\psi(X,U):=\phi(g(X,U))$; hence predictors based on $Z$ are a subset of predictors based on $(X,U)$ (and thus no better than those based on $X$), proving the risk inequality. \hfill$\square$

Proposition~\ref{prop:centralization} is not an endorsement of star in constrained settings.
It is a sanity check: \emph{hierarchies are only needed because some constraint makes full centralization infeasible or too expensive.}
In this paper, the dominant such constraint is the context window \(W\), which caps the fan-in of any node (Section~\ref{sec:budget-context}).
Once \(W\) binds, a star can only incorporate a bounded number of leaves regardless of budget, and that is precisely where hierarchy becomes a meaningful design choice.

\subsubsection{Why chains degrade without new evidence}
\label{sec:chains_degrade}

Chains are appealing in practice because they resemble deliberation: one agent's output becomes the next agent's input.
However, in our framework a chain does not introduce new independent evidence; it only retransmits and transforms what already exists.
Under lossy communication, repeated retransmission compounds loss.

This is explicit in both settings.
In the continuous warm-up, \eqref{eq:chain_mse} shows MSE increases linearly with chain length.
In the binary setting, the effect is even simpler.

\begin{proposition}[Bias decays along a chain under lossy communication]
\label{prop:chain_decay}
Consider a chain of length \(L\) in which a single leaf vote with bias \(\mu_0\) is transmitted through \(L\) independent communication steps, each with reliability \(\gamma(m)\), and no new independent evidence enters the chain.
Then the final bias is
\[
\mu_{\text{chain}}(L)=\gamma(m)^{L}\,\mu_0.
\]
Consequently, if \(\gamma(m)<1\) and \(L\ge 1\), then \(\mu_{\text{chain}}(L)<\mu_0\).
\end{proposition}

\paragraph{Proof.}
Each hop multiplies bias by \(\gamma(m)\) by the channel model in Section~\ref{sec:communication}. \hfill\(\square\)

\paragraph{Design takeaway.}
A relay-only chain strictly attenuates signal (Proposition~\ref{prop:chain_decay}); it helps only when intermediate steps add \emph{new} evidence (tools, retrieval, verification) rather than passing along the same information.
We therefore focus on star versus tree as the two scalable evidence-aggregation strategies under context limits.

\subsubsection{Context-limited star versus hierarchical trees}
\label{sec:context_limited_star_vs_tree}

The context window \(W\) turns Proposition~\ref{prop:centralization} from a theoretical upper bound into a practical obstacle.
A star aggregator that receives \(N\) messages of length \(m\) must read roughly \(Nm\) tokens, so feasibility requires
\begin{equation}
Nm \le W
\qquad\Longrightarrow\qquad
N \le N_{\max}(m):=\Big\lfloor \frac{W}{m}\Big\rfloor.
\label{eq:star_context_cap}
\end{equation}
For fixed \(W\), \(N_{\max}(m)\) is a \emph{constant} independent of budget.
Thus, once a star uses as many leaves as it can fit in context, additional budget cannot increase the number of incorporated leaves.
The only remaining levers are to strengthen each leaf (increase \(x\)) or to increase message length (increase \(m\)) at the cost of reducing \(N_{\max}(m)\).

A hierarchical tree avoids the global cap \eqref{eq:star_context_cap} by enforcing the context constraint \emph{locally}.
If each internal node has fan-in \(b\) and reads \(b\) messages of length \(m\), feasibility requires
\begin{equation}
bm \le W,
\label{eq:tree_local_context}
\end{equation}
but the number of leaves can grow as \(N=b^L\) with depth \(L\).
In other words, a tree converts a global bottleneck (\(Nm\le W\) at one node) into many local bottlenecks (\(bm\le W\) at many nodes).
Figure~\ref{fig:star_vs_tree} illustrates this trade-off: the star saturates early due to context limits, while the tree can continue scaling until it reaches its fixed-point ceiling \(\mu^\star\).

\begin{figure}[t]
  \centering
  \includegraphics[width=0.8\linewidth]{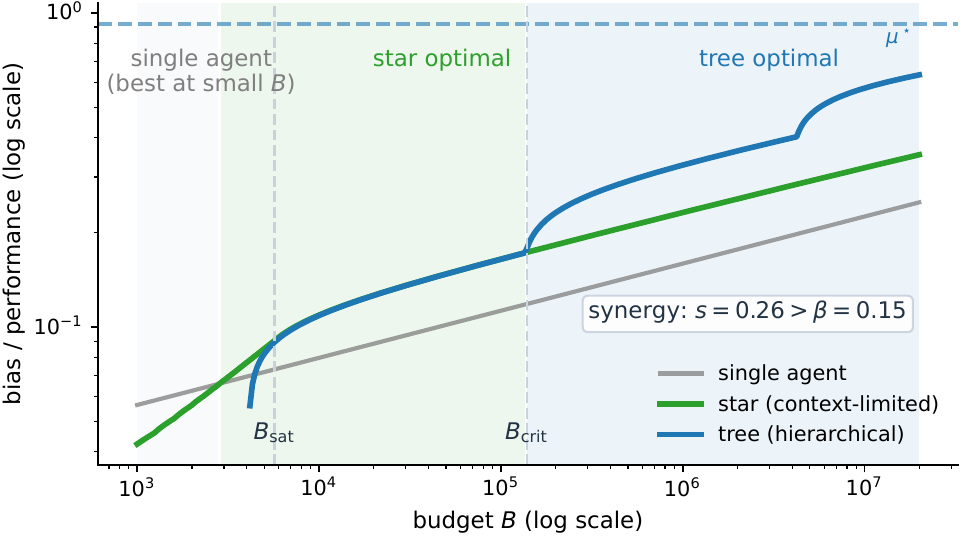}
  \caption{
  Star vs.\ tree vs.\ single-agent performance under a fixed context window $W$.
  A star can increase fan-in $N$ until it hits the context ceiling
  $N_{\max}=\lfloor W/m\rfloor$ (defining $B_{\mathrm{sat}}$), after which further budget
  can only improve per-leaf compute (scaling like $B^\beta$).
  A hierarchical tree can continue increasing the effective number of leaves by adding depth,
  exhibiting a growth exponent $s=\log(\alpha_\rho)/\log b$ in the growth regime; when $s>\beta$,
  trees eventually dominate beyond $B_{\mathrm{crit}}$.
  }
  \label{fig:star_vs_tree}
\end{figure}

This benefit is real only if information survives the hierarchy.
The binary phase transition from Section~\ref{sec:phase_transition} provides the feasibility condition:
a deep tree can only maintain nontrivial accuracy when
\begin{equation}
\alpha_\rho=\gamma(m)\big(\rho+(1-\rho)f_b'(0)\big) > 1.
\label{eq:supercritical_recall}
\end{equation}
Even when \(\alpha_\rho>1\), deeper trees saturate at a fixed point \(\mu^\star\), so trees should be viewed as a tool for \emph{scaling up to} a regime of strong performance, not for unlimited improvement.

\begin{figure}[!ht]
  \centering
  \includegraphics[width=0.7\linewidth]{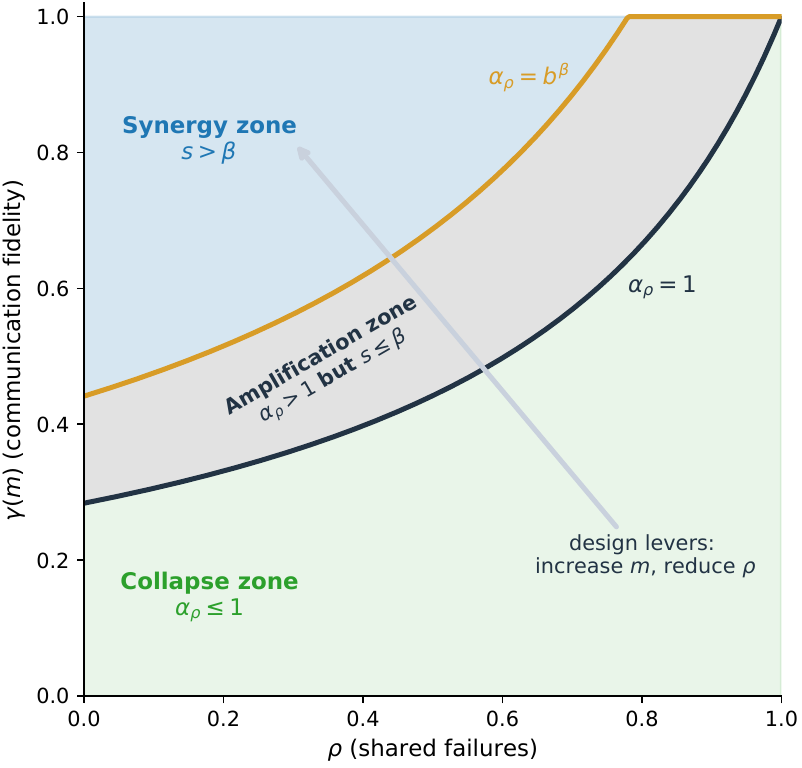}
  \caption{
  Phase regions in $(\rho,\gamma)$ for fixed fan-in $b$ and single-agent scaling exponent $\beta$.
  The boundary $\alpha_\rho=1$ separates collapse ($\alpha_\rho\le 1$) from amplification.
  The boundary $\alpha_\rho=b^\beta$ separates mere amplification (no synergy, $s\le\beta$)
  from synergy ($s>\beta$), where hierarchical scale-out can outpace single-agent scaling.
  Design levers such as increasing message length $m$ (raising $\gamma$) or reducing shared-failure
  correlation $\rho$ move systems toward the synergy region.
  }
  \label{fig:binary_phase_regions}
\end{figure}

Figure~\ref{fig:binary_phase_regions} summarizes the logic in a single picture.
A tree can only be useful if it lies above the amplification threshold \(\alpha_\rho>1\).
Among such trees, budgeted synergy further requires \(s>\beta\), which is equivalent to
\begin{equation}
s>\beta
\qquad \Longleftrightarrow\qquad
\alpha_\rho>b^\beta.
\label{eq:s_vs_beta_equiv}
\end{equation}
The next subsection turns \eqref{eq:s_vs_beta_equiv} into an explicit budget threshold.

\subsubsection{Budget thresholds for synergy: conditions on \texorpdfstring{$s$}{s} and \texorpdfstring{$\beta$}{beta}}
\label{sec:budget_thresholds}

The layer-wise analysis of Section~\ref{sec:binary_core} yields an organization exponent \(s\) that governs how signal grows with depth before saturation.
Here we translate that picture into a \emph{fixed-budget} comparison: under the same total budget \(B\), when can \textbf{scaling out} with a tree outperform \textbf{scaling up} a single agent?

\paragraph{Setup (growth regime).}
Assume the per-leaf scaling curve admits a local exponent \(\beta\) in the operating range (Section~\ref{sec:scaling}), so that \(\mu_0(x)=g(x)\approx kx^\beta\) when leaf bias is small.
In the amplification band of Theorem~\ref{thm:small_signal_band} (i.e., before the recursion approaches \(\mu^\star\)), a depth-\(L\) full \(b\)-ary tree satisfies
\begin{equation}
\mu_L \;\approx\; \mu_0(x)\,\alpha_\rho(b,m)^L
\;=\; \mu_0(x)\,N^{s(b,m,\rho)},
\qquad N=b^L,
\label{eq:growth_mu_tree}
\end{equation}
where \(\alpha_\rho(b,m)\) is defined in \eqref{eq:alpha_rho_def} and \(s(b,m,\rho)=\log\alpha_\rho(b,m)/\log b\) as in \eqref{eq:organization_exponent}.
With the tree budget model \(B\approx N\big(x+c_0(b,m)\big)\) from Section~\ref{sec:budget-context}, we have \(N\approx B/(x+c_0)\), giving the growth surrogate
\begin{equation}
\mu_{\mathrm{grow}}(B; b,m,x)
\;\approx\;
k\,x^\beta\left(\frac{B}{x+c_0(b,m)}\right)^{s(b,m,\rho)}.
\label{eq:grow_objective}
\end{equation}
We use \eqref{eq:grow_objective} only as a pre-saturation diagnostic; in Section~\ref{sec:saturation_binary} we clip predictions at the fixed point \(\mu^\star(b,m)\).

\begin{theorem}[Closed-form compute allocation in the growth regime]
\label{thm:opt_x_growth}
Fix \( (b,m,\rho) \) and suppose \( s(b,m,\rho)>\beta\).
Then the objective \eqref{eq:grow_objective} is maximized at
\begin{equation}
x^\star(b,m)
=
\frac{\beta}{s(b,m,\rho)-\beta}\,c_0(b,m),
\label{eq:xstar_growth}
\end{equation}
and the resulting optimized growth prediction scales as
\begin{equation}
\mu_{\mathrm{grow}}^\star(B;b,m)
\;\approx\;
k\,\kappa\big(s(b,m,\rho),\beta\big)\,c_0(b,m)^{\beta-s(b,m,\rho)}\,B^{s(b,m,\rho)},
\label{eq:mu_grow_star}
\end{equation}
where \(\kappa(s,\beta)=\beta^\beta(s-\beta)^{s-\beta}s^{-s}\).
If \(s(b,m,\rho)\le \beta\), then \eqref{eq:grow_objective} is (weakly) increasing in \(x\), so the growth surrogate favors \textbf{scale up} over \textbf{scale out}.
\end{theorem}

\paragraph{Proof sketch.}
The optimization is one-dimensional: maximize \(\log \mu_{\mathrm{grow}}=\beta\log x - s\log(x+c_0) + s\log B\).
Setting the derivative to zero yields \eqref{eq:xstar_growth}, and substituting gives \eqref{eq:mu_grow_star}.
This is the same calculus as in the continuous warm-up (Section~\ref{sec:cont_budget_tradeoff}), with the continuous exponent \(t\) replaced by the binary exponent \(s\); Appendix~\ref{app:xstar_growth_proof} gives the full derivation. \hfill\(\square\)

\paragraph{Synergy against a single agent.}
Spending the whole budget on a single agent yields
\begin{equation}
\mu_{\mathrm{single}}(B)
\;=\;
\mu_0(B)
\;\approx\;
k\,B^\beta.
\label{eq:single_scaling}
\end{equation}
Comparing \eqref{eq:mu_grow_star} with \eqref{eq:single_scaling} shows that exponent-level scale-out synergy in the growth regime requires \(s>\beta\) (equivalently \(\alpha_\rho>b^\beta\)).

\begin{corollary}[A budget threshold for scale-out synergy (growth regime)]
\label{cor:Bcrit}
Fix \((b,m,\rho)\) with \(s(b,m,\rho)>\beta\).
If
\begin{equation}
B \;\ge\; B_{\mathrm{crit}}(b,m)
\;\approx\; c_0(b,m)\,\kappa\big(s(b,m,\rho),\beta\big)^{-1/(s(b,m,\rho)-\beta)},
\label{eq:Bcrit}
\end{equation}
then the optimized growth prediction \eqref{eq:mu_grow_star} exceeds the single-agent scaling \eqref{eq:single_scaling}.
\end{corollary}

\paragraph{Interpretation.}
Equation~\eqref{eq:Bcrit} makes the trade-off explicit: synergy is harder when coordination is expensive (large \(c_0\)) or when \(s\) only barely exceeds \(\beta\).
Saturation can truncate the growth regime, but the threshold still indicates whether a budget window exists in which scale-out should help before the fixed point \(\mu^\star\) dominates.

\subsubsection{A universal upper bound on amplification exponents}
\label{sec:upper_bound_s}

The organization exponent \(s\) governs how quickly a majority-vote hierarchy can amplify weak signal as the number of leaves grows.
It is therefore important to understand its fundamental limits.
Perhaps surprisingly, the majority family admits a clean and universal upper bound: no choice of fan-in can make \(s\) exceed \(1/2\) in the binary model with one-bit messages.

\begin{theorem}[A universal bound: \(s \le \tfrac{1}{2}\)]
\label{thm:s_upper_half}
For every odd \(b\ge 3\), the majority map satisfies
\begin{equation}
f_b'(0)\le \sqrt{b}.
\label{eq:fbprime_sqrtb}
\end{equation}
Consequently, for any \(\rho\in[0,1)\) and any \(\gamma(m)\le 1\),
\begin{equation}
\alpha_\rho=\gamma(m)\big(\rho+(1-\rho)f_b'(0)\big)\le \sqrt{b},
\qquad\text{and hence}\qquad
s=\frac{\log\alpha_\rho}{\log b}\le \frac{1}{2}.
\label{eq:s_le_half}
\end{equation}
\end{theorem}

\paragraph{Proof sketch.}
Equation~\eqref{eq:fbprime_sqrtb} follows from the closed form \(f_b'(0)=\frac{b\binom{b-1}{(b-1)/2}}{2^{b-1}}\) (Lemma~\ref{lem:fb_props}) and a standard upper bound on the central binomial coefficient.
Plugging \eqref{eq:fbprime_sqrtb} into \eqref{eq:alpha_rho_def} yields \eqref{eq:s_le_half}.
A fully elementary proof is given in Appendix~H. \hfill\(\square\)

\paragraph{Phase-diagram summary and implications.}
The preceding results collapse topology choice into a small set of quantitative checks.
Under context constraints, deep trees are only viable in the supercritical regime \(\alpha_\rho(b,m)>1\) (otherwise bias collapses to chance as depth grows).
In the growth regime, budgeted synergy against scale-up further requires \(s(b,m,\rho)>\beta\), equivalently \(\alpha_\rho>b^\beta\) (Section~\ref{sec:budget_thresholds}).
Within one-bit majority-vote hierarchies we also have the universal cap \(s\le 1/2\) (Theorem~\ref{thm:s_upper_half}), which explains why scale-out is fragile on some tasks: if \(\beta\ge 1/2\), exponent-level synergy is impossible without richer communication, heterogeneity that lowers the effective \(\rho\), or protocols that change the effective aggregation map.
Even when \(s>\beta\), gains persist only until saturation at \(\mu^\star(b,m)\); beyond that point, additional budget is typically better spent improving fidelity (larger \(m\), hence larger \(\gamma(m)\)) or reducing shared failures (smaller \(\rho\)) than on additional depth.
Figure~\ref{fig:binary_phase_regions} summarizes these regimes, and Section~\ref{sec:design} turns the checks into concrete design diagnostics within the model.

\section{Design Diagnostics Within the Model}
\label{sec:design}

The results so far characterize when different organizations can help under a fixed budget: a context-limited star saturates (Section~\ref{sec:context_limited_star_vs_tree}), chains compound loss (Section~\ref{sec:chains_degrade}), and trees are viable only in the supercritical regime \(\alpha_\rho>1\) with small-signal growth governed by an organization exponent \(s\) (Section~\ref{sec:binary_core}).
Here we translate these results into \emph{design diagnostics within the model}: given a budget \(B\), a context window \(W\), and a small set of effective environment quantities, we ask whether hierarchy is feasible, whether scale-out can outpace scale-up before saturation, and which constraint is binding when it cannot.
For LLM-based agent systems, the knobs correspond to token allocations: per-leaf reasoning/compute \(x\) and inter-agent message length \(m\). The abstraction is deliberately coarse, so the goal is regime prediction and diagnostic guidance, not a faithful end-to-end optimizer for any particular agent stack.

Under a budget \(B\) and context window \(W\), we parameterize a tree design by
\[
(b,\ m,\ x,\ L)\quad \text{(fan-in, message length, per-leaf compute, depth)}
\]
and ask how to choose these quantities to maximize the model-predicted performance subject to feasibility.
The inner problem chooses \(x\) for fixed \((b,m)\); the outer problem trades off communication quality and coordination cost by selecting \((b,m)\) as a function of \(B\).
A key structural takeaway is monotonicity: in the growth regime, the budget-optimal message length can be chosen nondecreasing in \(B\), yielding efficient design curves rather than expensive search.
We focus on binary success/failure tasks, where hierarchies can either amplify or collapse; the continuous setting admits analogous (and smoother) diagnostics via the closed-form recursions of Section~\ref{sec:continuous_warmup}.

\subsection{Compute allocation in the growth regime}
\label{sec:design_xstar}

Fix a feasible fan-in \(b\) and message length \(m\) with \(bm\le W\).
Let \(\alpha_\rho(b,m)\) be defined as in \eqref{eq:alpha_rho_def}, and write the associated organization exponent as \(s(b,m):=\log\alpha_\rho(b,m)/\log b\).
In the supercritical regime \(\alpha_\rho(b,m)>1\) and before saturation, the binary small-signal prediction takes the growth form \eqref{eq:grow_objective}.
Theorem~\ref{thm:opt_x_growth} yields a closed-form compute allocation \(x^\star(b,m)\) (eq.~\eqref{eq:xstar_growth}) and an optimized growth curve \(\mu_{\mathrm{grow}}^\star(B;b,m)\) (eq.~\eqref{eq:mu_grow_star}) whenever \(s(b,m)>\beta\).
If \(s(b,m)\le \beta\), the growth surrogate favors spending budget to \textbf{scale up} rather than \textbf{scale out}.

Given a budget \(B\), we then allocate
\begin{equation}
N^\star(B;b,m)\approx \left\lfloor \frac{B}{x^\star(b,m)+c_0(b,m)} \right\rfloor,
\qquad
L^\star(B;b,m)\approx \left\lfloor \log_b N^\star(B;b,m) \right\rfloor,
\label{eq:design_NL}
\end{equation}
and cap \(L^\star\) using the mixing-depth guarantee (Theorem~\ref{thm:mixing_depth}) when operating near saturation.

\subsection{Monotone communication design: why \texorpdfstring{$m^\star(B)$}{m*(B)} increases with budget}
\label{sec:design_monotone_m}

Compute allocation solves only the inner problem.
We still need to choose message length \(m\), which trades off two effects:
increasing \(m\) improves communication fidelity \(\gamma(m)\) and thus increases the gain \(\alpha_\rho\) (raising the exponent \(s\)),
but increasing \(m\) also increases the coordination cost \(c_0(b,m)\), reducing the number of leaves we can afford.

A natural design question is whether there is structure in the budget-dependent solution.
The answer is yes:
\emph{in the growth regime, the optimal message length is monotone in the budget.}
This formalizes a useful heuristic in the model: as you spend more, the optimal design communicates less aggressively compressed information.

\paragraph{A growth-regime objective indexed by \(m\).}
Fix fan-in \(b\) and consider the feasible message set
\[
\mathcal{M}_b := \{m\in\mathbb{Z}_+:\ bm\le W\}.
\]
For each \(m\in\mathcal{M}_b\), define the growth-optimal prediction from Theorem~\ref{thm:opt_x_growth},
\begin{equation}
\underline{\mu}_b(m;B)
:=
\mu_{\mathrm{grow}}^\star(B;b,m)
=
k\,\kappa(s_{b,m},\beta)\,c_0(b,m)^{\beta-s_{b,m}}\,B^{s_{b,m}},
\qquad
s_{b,m}:=s(b,m).
\label{eq:mu_bm}
\end{equation}
We only consider candidates with \(s_{b,m}>\beta\) and \(\alpha_\rho(b,m)>1\) since only these can outpace scale-up in exponent terms in the growth regime.

\begin{theorem}[Monotone message-length design in the growth regime]
\label{thm:monotone_m}
Fix fan-in \(b\ge 3\) and correlation \(\rho\in[0,1)\).
Assume \(\gamma(m)\) is strictly increasing in \(m\) and \(c_0(b,m)\) is nondecreasing in \(m\).
Let
\[
m_b^\star(B)\in \arg\max_{m\in\mathcal{M}_b}\ \underline{\mu}_b(m;B),
\]
where the maximization ranges over candidates satisfying \(s_{b,m}>\beta\) and \(\alpha_\rho(b,m)>1\).
Then for any \(B_2>B_1\), there exists an optimal selection such that
\[
m_b^\star(B_2)\ \ge\ m_b^\star(B_1).
\]
Equivalently, \(m_b^\star(B)\) can be chosen as a nondecreasing function of \(B\).
\end{theorem}

\paragraph{Why this holds.}
Taking logs of \eqref{eq:mu_bm} yields
\[
\log \underline{\mu}_b(m;B)= s_{b,m}\,\log B + q_b(m),
\]
where \(q_b(m)=\log k+\log \kappa(s_{b,m},\beta)+(\beta-s_{b,m})\log c_0(b,m)\) does not depend on \(B\).
Thus, for fixed \(b\), each message length \(m\) defines a line in \(\log B\) with slope \(s_{b,m}\).
Because \(\gamma(m)\) increases with \(m\), the slopes \(s_{b,m}\) are increasing in \(m\).
When choosing among lines with increasing slopes, the maximizing line index can only move in one direction as \(\log B\) increases.
A full proof, stated in terms of increasing differences and upper envelopes, is given in Appendix~I.

\begin{figure}[!ht]
  \centering
  \begin{subfigure}[t]{0.485\linewidth}
    \centering
    \includegraphics[width=\linewidth]{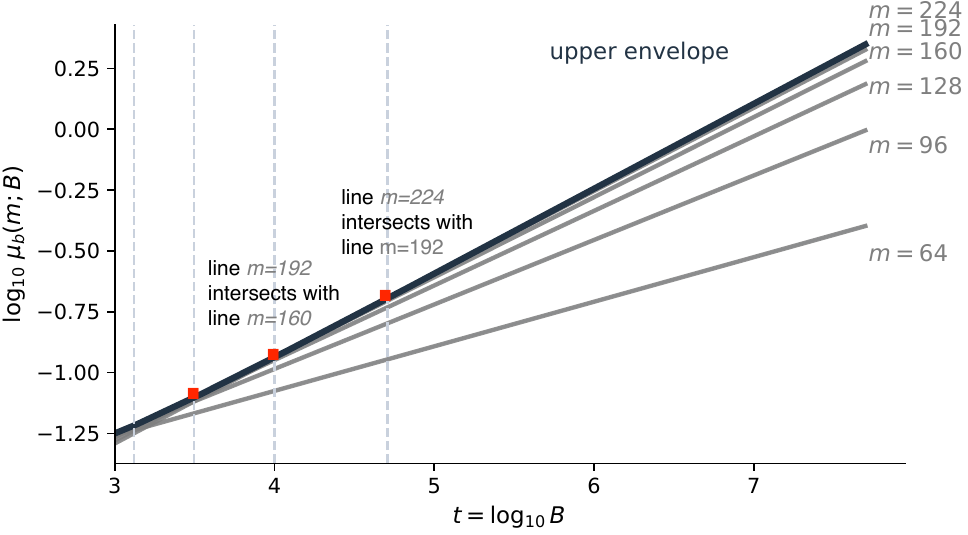}
    \caption{Upper envelope over candidate message lengths.}
  \end{subfigure}
  \hfill
  \begin{subfigure}[t]{0.485\linewidth}
    \centering
    \includegraphics[width=\linewidth]{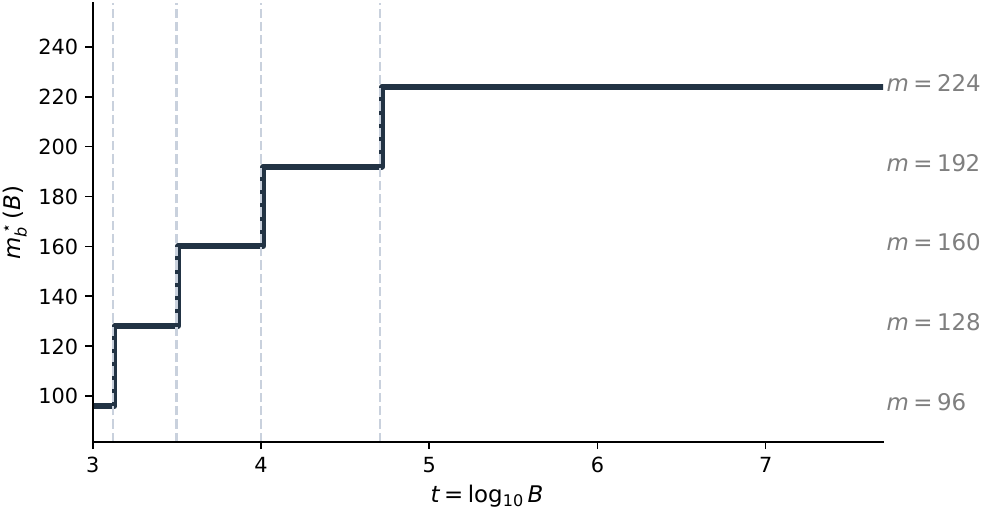}
    \caption{Resulting $m_b^\star(B)$ (monotone step function).}
  \end{subfigure}
\caption{Upper-envelope argument behind the monotone message-length design curve (fixed fan-in $b$).
For each feasible message length $m$ (numbers shown are tokens in this example), we first optimize the per-leaf compute allocation $x$ in the growth-regime surrogate, yielding a closed-form prediction of the form
$\log \underline{\mu}_b(m;B)= s_{b,m}\,\log B + q_b(m)$,
i.e., a line in $t=\log B$ whose slope $s_{b,m}$ increases with $m$ because communication fidelity $\gamma(m)$ improves with longer messages.
\textbf{(a)} We plot these candidate lines for a representative subset of feasible $m$ values and highlight their pointwise maximum (thick curve), i.e., the upper envelope.
Red markers indicate budgets where two candidates tie (line intersections), and the vertical dashed lines project these switch points onto the $t$-axis.
\textbf{(b)} The maximizing message length $m_b^\star(B)$ is therefore a nondecreasing step function of $B$, jumping exactly at the envelope breakpoints.
Algorithm~\ref{alg:envelope} constructs the full envelope efficiently by scanning candidates in increasing $m$ (increasing slope) and maintaining only the active lines and their intersection thresholds.}
\label{fig:envelope_m}
\end{figure}

\paragraph{What the theorem does and does not claim.}
Theorem~\ref{thm:monotone_m} is a statement about the \emph{growth-regime} objective \(\underline{\mu}_b(m;B)\) for a fixed fan-in \(b\).
It does not assert global monotonicity when the optimal fan-in \(b\) itself changes with budget, nor does it rely on the saturation fixed point \(\mu^\star\).
In practice, once a design approaches saturation, increasing \(m\) often remains beneficial because it increases \(\alpha_\rho\) and improves \(\mu^\star\), but the theorem is intentionally limited to a regime where the approximation is provably accurate.

\subsection{A linear-time envelope algorithm for design curves}
\label{sec:envelope_algorithm}

The envelope view in Figure~\ref{fig:envelope_m} yields an efficient algorithm for generating design curves.
For a fixed fan-in \(b\), we want to compute \(m_b^\star(B)\) for a range of budgets, without brute-force search over all \(m\) at every \(B\).

\paragraph{Lines and intersections.}
Let \(t=\log B\).
For each feasible \(m\in\mathcal{M}_b\) with \(s_{b,m}>\beta\) and \(\alpha_\rho(b,m)>1\), define
\begin{equation}
\ell_{b,m}(t) := s_{b,m}\,t + q_b(m),
\label{eq:line_def}
\end{equation}
where \(q_b(m)\) is the intercept from \eqref{eq:mu_bm}.
The maximizer of \(\ell_{b,m}(t)\) over \(m\) is the envelope.
Because \(s_{b,m}\) is increasing in \(m\), the envelope can be computed by a standard monotone upper-hull procedure:
maintain a stack of candidate lines and the \(t\)-values where one line overtakes the previous one.

Given two lines \(\ell_1(t)=s_1 t+q_1\) and \(\ell_2(t)=s_2 t+q_2\) with \(s_1<s_2\), their intersection occurs at
\begin{equation}
t^\star = \frac{q_1-q_2}{s_2-s_1}.
\label{eq:intersection}
\end{equation}
If a newly added line overtakes the previous line earlier than the previous line overtook its predecessor, then the previous line is never optimal and can be removed.
This is the same geometric logic used in convex-hull tricks for dynamic programming.

\begin{algorithm}[t]
\caption{Envelope construction for \(m_b^\star(B)\) (fixed \(b\)).}
\label{alg:envelope}
\small
\begin{algorithmic}[1]
\State \textbf{Enumerate candidates (in increasing \(m\)).}
\State \(M \gets \lfloor W/b \rfloor\). For \(m=1,2,\dots,M\): compute \(\alpha_\rho(b,m)\) and \(s_{b,m}\).
\State Keep \(m\) only if \(\alpha_\rho(b,m)>1\) and \(s_{b,m}>\beta\). For each kept \(m\), form \(\ell_{b,m}(t)=s_{b,m}t+q_b(m)\).

\State \textbf{(Optional robustness) Remove equal-slope duplicates.}
\State If two retained candidates have the same slope \(s\), keep only the one with larger intercept \(q_b(m)\).

\State \textbf{Build the upper envelope (stack).}
\State Maintain a stack of \((m,\ell,\tau)\), where \(\tau\) is the activation time.
\For{retained \(m\) in increasing order}
  \If{stack empty}
    \State push \((m,\ell_{b,m},-\infty)\)
  \Else
    \State let \((m_{\text{top}},\ell_{\text{top}},\tau_{\text{top}})\) be stack top
    \State compute intersection time \(t^\star\) using \(t^\star = \frac{q_{\text{top}}-q_b(m)}{s_{b,m}-s_{\text{top}}}\)
    \While{stack nonempty and \(t^\star \le \tau_{\text{top}}\)}
      \State pop stack; update \((\ell_{\text{top}},\tau_{\text{top}})\); recompute \(t^\star\)
    \EndWhile
    \State push \((m,\ell_{b,m},t^\star)\)
  \EndIf
\EndFor

\State \textbf{Query (one pass over budgets).}
\State For budgets \(B_1<\dots<B_T\), set \(t_i=\log B_i\).
\State Walk a pointer through breakpoints and output the active \(m\) for each \(t_i\).

\end{algorithmic}
\end{algorithm}

\paragraph{Complexity and usage.}
For a fixed fan-in \(b\), the feasible message set has size \(|\mathcal{M}_b|=\lfloor W/b\rfloor\).
Algorithm~\ref{alg:envelope} scans candidates in increasing \(m\) (equivalently increasing slope \(s_{b,m}\)) and builds the upper envelope with a stack in which each candidate line is pushed once and popped at most once.
Thus, envelope construction takes \(O(|\mathcal{M}_b|)\) time and \(O(|\mathcal{M}_b|)\) memory.
Given budgets \(B_1<\cdots<B_T\), the optimal choices \(m_b^\star(B_t)\) can then be returned in \(O(T)\) additional time by a single pass over the envelope breakpoints.
To obtain a global design curve, we run Algorithm~\ref{alg:envelope} for each odd \(b\ge 3\) with \(\mathcal{M}_b\neq\emptyset\) (equivalently \(b\le W\) since \(m\ge 1\)), and then select the best \((b,m)\) per budget using the clipped objective described next.

\subsection{Putting the diagnostics together}
\label{sec:design_checklists}

We now combine compute allocation, monotone communication design, and saturation into a compact workflow.
The inputs are a budget \(B\), a context window \(W\), and a small set of estimated effective quantities, in particular \(\mu_0(x)\), \(\rho\), and \(\gamma(m)\) (Section~\ref{sec:measurable-params}).
The output is a model-based prediction and a set of diagnostics (which constraint is binding), rather than a claim that the abstraction optimizes real agent stacks end-to-end. 
Figure~\ref{fig:design_pipeline} summarizes the workflow of design-diagnostics.

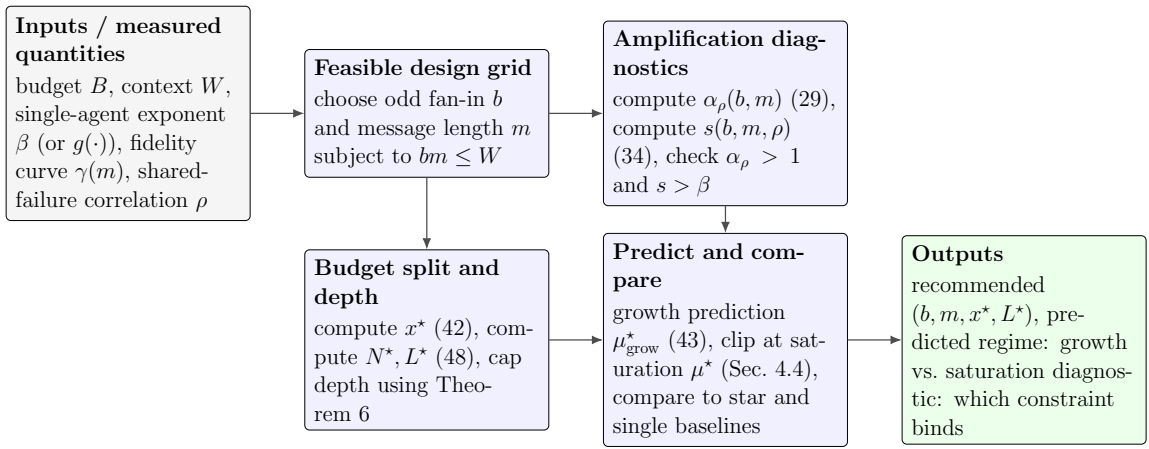
\begin{figure}[!ht]
\centering
\resizebox{\textwidth}{!}{%
\begin{tikzpicture}[font=\Large, >=Latex, node distance=1.6cm and 1.2cm]

\tikzset{
  base/.style={draw, rounded corners=3pt, align=left, inner sep=6pt,
               text width=5.0cm, minimum height=2.8cm},
  inputbox/.style={base, fill=gray!8},
  procbox/.style={base, fill=blue!6},
  outbox/.style={base, fill=green!8},
  arr/.style={-{Latex[length=2.5mm, width=2mm]}, thick, color=black!70}
}

\node[inputbox] (inputs)
  {\textbf{Inputs / measured quantities}\\[2pt]
   budget \(B\), context \(W\),
   single-agent exponent \(\beta\) (or \(g(\cdot)\)), fidelity curve \(\gamma(m)\),
   shared-failure correlation \(\rho\)};

\node[procbox, right=of inputs] (grid)
  {\textbf{Feasible design grid}\\[2pt]
   choose odd fan-in \(b\) and
   message length \(m\)  subject to \(bm\le W\)};

\node[procbox, right=of grid] (exponents)
  {\textbf{Amplification diagnostics}\\[2pt]
   compute \(\alpha_\rho(b,m)\) \eqref{eq:alpha_rho_def},
   compute \(s(b,m,\rho)\) \eqref{eq:organization_exponent},
   check \(\alpha_\rho>1\) and \(s>\beta\)};

\node[procbox, below=of grid] (alloc)
  {\textbf{Budget split and depth}\\[2pt]
   compute \(x^\star\) \eqref{eq:xstar_growth},
   compute \(N^\star,L^\star\) \eqref{eq:design_NL},
   cap depth using Theorem~\ref{thm:mixing_depth}};

\node[procbox, right=of alloc] (predict)
  {\textbf{Predict and compare}\\[2pt]
   growth prediction \(\mu^\star_{\mathrm{grow}}\) \eqref{eq:mu_grow_star},
   clip at saturation \(\mu^\star\) (Sec.~\ref{sec:saturation_binary}),
   compare to star and single baselines};

\node[outbox, right=of predict] (output)
  {\textbf{Outputs}\\[2pt]
   recommended \((b,m,x^\star,L^\star)\),
   predicted regime: growth vs.\ satu\-ration diagnostic: which constraint binds};

\draw[arr] (inputs.east) -- (grid.west);
\draw[arr] (grid.east) -- (exponents.west);
\draw[arr] (grid.south) -- (alloc.north);
\draw[arr] (exponents.south) -- (predict.north);
\draw[arr] (alloc.east) -- (predict.west);
\draw[arr] (predict.east) -- (output.west);

\end{tikzpicture}
}
\caption{A design-diagnostics pipeline within the model.  Measured constraints (\(B,W\)) and calibrated parameters (\(\beta\), \(\gamma(m)\), \(\rho\)) determine which \((b,m)\) are feasible, whether depth amplifies (\(\alpha_\rho>1\)), and whether scale-out can beat scale-up in the growth regime (\(s>\beta\)).  The closed-form allocation \(x^\star\) and implied \((N^\star,L^\star)\) yield a predicted performance that is clipped at the saturation fixed point \(\mu^\star\) and compared against star and single-agent baselines.}
\label{fig:design_pipeline}
\end{figure}

\paragraph{Compare against centralization.}
For message length \(m\), a star can include at most \(N_{\max}(m)=\lfloor W/m\rfloor\) leaves.
When the budget regime of interest does not push beyond this cap, centralization is the natural baseline: it aggregates in one hop and avoids multi-hop loss (Proposition~\ref{prop:centralization}).
When \(N_{\max}(m)\) binds, additional budget cannot increase the number of incorporated leaves in a star, and hierarchy becomes the only way to expose more parallel evidence under the same per-node context limit.

\paragraph{Screen tree candidates and allocate compute.}
For each odd fan-in \(b\ge 3\) and each feasible \(m\) with \(bm\le W\), compute
\(\alpha_\rho(b,m)\) and \(s(b,m)\).
Discard candidates with \(\alpha_\rho(b,m)\le 1\), which collapse in depth (Theorem~\ref{thm:phase_transition}).
For exponent-level budgeted synergy in the growth regime, also discard \(s(b,m)\le\beta\).
For remaining candidates, set the per-leaf compute to the closed form \(x^\star(b,m)\) in \eqref{eq:xstar_growth} (Theorem~\ref{thm:opt_x_growth}) and choose the implied leaf count and depth via \eqref{eq:design_NL}, with depth interpreted through the mixing-depth guarantee of Theorem~\ref{thm:mixing_depth}.
For fixed \(b\), Algorithm~\ref{alg:envelope} yields the monotone budget-indexed communication curve \(m_b^\star(B)\), so the outer search can be performed efficiently.

\paragraph{Clip by saturation.}
In the supercritical regime, the binary recursion converges to a fixed point \(\mu^\star(b,m)\) (computed by iterating \(u\mapsto f_{b,\rho}(\gamma(m)u)\)).
To avoid extrapolating the growth law beyond its range, we evaluate a design using the clipped prediction from Section~\ref{sec:saturation_binary}:
\begin{equation}
\widehat{\mu}(B;b,m)
:=
\min\big\{\mu^\star(b,m),\ \mu_0(x)\,\alpha_\rho(b,m)^L\big\}.
\label{eq:design_clipped}
\end{equation}
At each budget \(B\), we compare this clipped tree prediction to the best feasible star and to the single-agent baseline under the same budget.

\paragraph{Interpretation as diagnostics.}
These checks separate four common reasons why scaling out fails: context saturation of a star (\(N_{\max}(m)\) caps centralization), subcriticality (\(\alpha_\rho\le 1\)), exponent mismatch (\(s\le\beta\)), and saturation (\(\widehat{\mu}\) close to \(\mu^\star\)).
When performance plateaus, the theory points to two levers: improve effective communication fidelity (larger \(\gamma(m)\) at the same cost) or reduce shared failures (smaller \(\rho\) via diversity or verification), rather than increasing depth.

\paragraph{Continuous tasks.}
The same workflow applies in the continuous warm-up with \(\gamma(m)\) replaced by \(\sigma_c^2(m)\) and \(\mu^\star\) replaced by the explicit error floor \(v^\star\) in \eqref{eq:tree_floor}.


\section{Empirical Touchpoints}
\label{sec:experiments}

This paper is primarily theoretical: our goal is to explain \emph{why} budgeted multi-agent scaling amplifies in some regimes and collapses in others, and to make the relevant regime boundaries explicit.
Accordingly, we keep the empirical component modest.
We (i) connect our bottlenecks and regime predictions to recent large-scale matched-budget evaluations reported elsewhere, and (ii) provide a minimal synthetic sanity check that recovers the predicted amplification boundary under the assumed model.

\subsection{Theoretical Mechanisms in Recent Large-Scale Studies}
\label{sec:external_evidence}

Large-scale matched-budget evaluations are computationally prohibitive for academic groups, but recent industrial studies provide high-fidelity data that validates our theoretical predictions. 
Most notably, \citet{kim2025scalingagents} conducted an extensive evaluation of scaling laws for agent systems across multiple benchmarks and architectures. 
Their empirical findings align closely with the qualitative regime boundaries derived in our framework:
\begin{itemize}
  \item \textbf{Context saturation and coordination overhead.} \citet{kim2025scalingagents} observe that simply concatenating agent traces can degrade performance on tool-heavy tasks when coordination overhead consumes the context window. In our framework, this corresponds to the $Nm\le W$ bottleneck: once a centralized aggregator becomes context-saturated, adding agents can dilute usable signal.
  \item \textbf{Saturation at strong baselines.} They report diminishing (and sometimes negative) returns once the single-agent baseline is already strong. This is consistent with our fixed-point analysis: once a hierarchy approaches its saturation point $\mu^\star$, marginal gains from adding depth or leaves are small; improvement requires changing effective parameters (better $\gamma(m)$ or smaller $\rho$), not more aggregation.
  \item \textbf{Shared failures and error cascades.} They emphasize that agents can share failure modes and that naive aggregation can amplify errors, while more structured coordination can mitigate this. In our framework, shared failures are summarized by $\rho$: higher $\rho$ pushes $\alpha_\rho$ downward and can drive hierarchies into the subcritical regime where depth washes out information.
\end{itemize}

By providing the mathematical derivation for these empirically observed regimes, our theory explains \emph{why} scaling fails when it does, suggesting that future improvements must come from altering the effective parameters, specifically increasing $\gamma(m)$ or decreasing $\rho$, rather than simply adding more agents.

\subsection{Synthetic sanity check: recovering the amplification boundary}
\label{sec:synthetic_experiments}

Synthetic experiments provide a controlled setting where the ground-truth parameters are known and the phase boundary can be probed directly.
Our goal here is not to optimize absolute performance, but to verify that the qualitative prediction, i.e., amplification versus collapse governed by \(\alpha_\rho\), is recovered under the generative model analyzed in Section~\ref{sec:binary_core}.

We include one lightweight Monte Carlo study under the \(\rho\)-shared binary model.
We fix \(b=5\), \(\gamma=0.6\), \(\mu_0=0.05\), sweep \(\rho\in[0,0.8]\), and compare the recursion prediction \(\mu_L\) at depths \(L\in\{10,30\}\) to Monte Carlo estimates (\(n=50{,}000\)).
Figure~\ref{fig:synthetic_phase_transition} shows close agreement and a sharpening transition around the predicted boundary \(\alpha_\rho=1\).

\begin{figure}[t]
  \centering
  \includegraphics[width=0.92\linewidth]{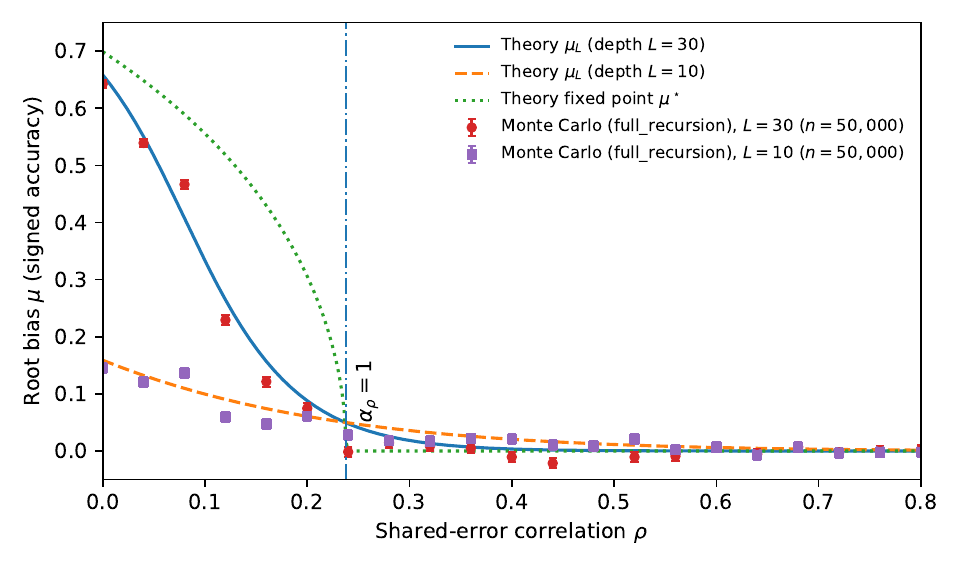}
  \caption{Synthetic Monte Carlo sanity check for the binary majority tree under the \(\rho\)-shared model and a binary symmetric channel with reliability \(\gamma\).
We plot the root bias \(\mu_L\) as a function of \(\rho\) for \(b=5\), \(\gamma=0.6\), and \(\mu_0=0.05\).
Curves show recursion predictions at finite depth; markers show Monte Carlo estimates (\(n=50{,}000\)).
The vertical line indicates the threshold \(\alpha_\rho=1\), separating amplification (\(\alpha_\rho>1\)) from collapse (\(\alpha_\rho\le 1\)) in the deep limit.}
  \label{fig:synthetic_phase_transition}
\end{figure}

For continuous tasks, the same role can be played by a simple correlated-Gaussian generator, where the closed-form expressions of Section~\ref{sec:cont_closed_form} make correlation and communication floors explicit.

\section{Discussion and Limitations}
\label{sec:discussion}

We asked when scaling out a budgeted multi-agent system yields synergy and when it fails.
The framework is intentionally minimal: it treats an agent as a black-box solver with an allocated compute budget, models inter-agent communication as a controllable lossy channel, and summarizes shared failures through an effective correlation parameter.
This minimalism makes the regime boundaries transparent, but it also defines what the theory does and does not claim.

While our results provide a rigorous asymptotic foundation for agent scaling, the practical utility of these bounds hinges on the empirical estimation of the shared-failure correlation $\rho$ and the organization exponent $s$. In practice, $\rho$ serves as a `homogeneity metric' for a given model ensemble; a high $\rho$ suggests that the agents are inheriting the same inductive biases or training data flaws, thereby collapsing the effective fan-in $b$. Future work should focus on developing diagnostic probes—short, high-entropy test batteries—to calibrate these parameters for specific LLM families. By doing so, the Phase Transition $\alpha_{\rho}$ can be transformed from a theoretical threshold into a real-time `stopping rule' for determining whether further tree depth will yield diminishing returns or catastrophic signal decay.

\paragraph{Relationship to prior results}
Our binary phase transition echoes reconstruction thresholds for broadcasting on trees \citep{kesten1966additional,evans2000broadcasting,mossel2001reconstruction}, where information can or cannot propagate from leaves to the root depending on a single scalar.
Two differences are central for agent-system design.
First, our ``channel'' is partly under design control: message length trades off fidelity \(\gamma(m)\) against budget and context constraints, so the relevant threshold depends on an explicit communication choice.
Second, we couple propagation to a fixed-budget allocation problem: an amplifiable hierarchy is synergistic only when its organization exponent outpaces single-agent scaling.

Recent work offers complementary perspectives on information flow in agent systems.
\citet{he2025infoagents} cast agentic architectures as compressor--predictor pairs and propose mutual information between context and compression as a task-independent proxy for information survival; 
their finding that scaling the compressor can dominate scaling the predictor aligns with our emphasis on communication fidelity \(\gamma(m)\) as a first-order bottleneck, though our framework adds explicit budget constraints and multi-hop aggregation structure.
\citet{li2026skillselection} study skill selection in single-agent systems and observe a phase transition: selection accuracy remains stable up to a critical library size then drops sharply, with semantic confusability among skills playing a central role.
This echoes our amplification--collapse dichotomy, but at a different interface: where we model information loss across agents, they model selection confusion within a single agent's skill routing.
Together, these results suggest that sharp thresholds, whether in inter-agent communication, intra-agent skill selection, or hierarchical aggregation, may be a recurring motif in bounded-resource AI systems.

\paragraph{What the framework explains (and what it does not)}
The framework is designed to explain qualitative regime changes that are otherwise hard to diagnose: why stars saturate under context limits, why chains compound loss, why trees can either amplify or collapse depending on \(\alpha_\rho\), and why correlated failures impose floors that make large ensembles behave like ``one agent with shared blind spots.''
It also clarifies what additional sophistication must accomplish: any useful complexity must effectively increase communication fidelity, reduce shared failures, or relax the fan-in constraint.

At the same time, the model is not a faithful simulator of real agent stacks.
Parameters such as \(\rho\) and \(\gamma(m)\) are effective summaries that can vary with depth, prompts, tools, and task difficulty; real communication is far richer than a one-bit abstraction; and practical budget accounting can differ across implementations.
For these reasons, we view the theory as a coordination diagnostic, i.e., predicting which constraint is binding and which direction should help, rather than as a source of precise point predictions.
Crucially, the alignment between our derived bottlenecks and independent large-scale observations \citep{kim2025scalingagents} suggests that the effective parameters defined here capture the structural essence of the scaling problem, rendering the theory robust to changes in the underlying base models.

\paragraph{How stronger protocols fit the framework}
Protocols such as debate, critique, verification, or self-consistency can be interpreted as mechanisms that change the effective parameters of the model \citep{irving2018debate,christiano2018amplifying,du2024multiagentdebate,wang2023selfconsistency}.
For example, structured verification can reduce effective shared failures (\(\rho\)), while richer message schemas can improve effective fidelity (\(\gamma(m)\)); both, however, typically increase coordination cost.
The theory suggests a concrete way to evaluate such protocols: measure how much they move the effective bottlenecks per unit budget, and check whether they move the system across the supercritical boundary \(\alpha_\rho>1\) or the synergy boundary \(s>\beta\).

\paragraph{Extensions beyond our scope}
Several extensions are natural.
Allowing heterogeneity across leaves (different models, prompts, or tools) makes \(\rho\) and scaling task- and agent-dependent; a tractable next step is to treat \(\rho\) as depth-dependent or to model mixtures of agent types.
More general graphs and multi-round interaction can be analyzed by unrolling protocols into computation graphs, but doing so requires accounting for how added edges change effective dependence as well as cost.
Finally, our strongest binary bounds rely on one-bit-style messages and majority aggregation; allowing richer messages or confidence reports changes the amplification map and may relax universal exponent limits.
In all cases, the guiding principle remains the same: additional structure helps only insofar as it relaxes the binding constraints (context, communication loss, or shared failures) under a fixed budget.

\section{Conclusion}
\label{sec:conclusion}

Multi-agent systems can exhibit synergy, but under a fixed budget scaling out is not monotone: it can help, saturate, or collapse.
We argued that these outcomes become predictable once we make the binding constraints explicit: finite context windows, lossy communication, and shared failures, and summarize an environment by a small set of measurable effective parameters.
In a continuous warm-up we derived closed-form risks for star/chain/tree organizations and exposed correlation and communication floors.
For binary success/failure tasks, we proved a sharp amplification--collapse phase transition for deep majority trees and characterized the resulting fixed-point saturation and mixing depth.
In the amplifying regime we introduced an organization exponent \(s\) governing small-signal growth and showed that budgeted synergy in the growth range occurs exactly when \(s>\beta\), yielding closed-form compute allocation and explicit budget thresholds.
We also translated these results into simple design diagnostics within the model.
The broader message is that topology is not magic.
Hierarchies are useful because context limits can make full centralization infeasible, and they succeed only when communication is sufficiently faithful and shared failures are sufficiently weak.
We hope this work helps move agent-system design from ad hoc heuristics toward principled, testable coordination theories.

\clearpage
\acks{The authors gratefully acknowledge funding from Canada CIFAR AI
Chair Program and the Canada NSERC Discovery Grant program.}
\appendix

\section*{Proof Roadmap and Dependency Graph}
This appendix provides self-contained proofs for the main theoretical statements in the paper.
For convenience, we list where each result is proved:
\begin{itemize}
  \item Lemma~\ref{lem:fb_props} (properties of the majority map): Appendix~A.
  \item Lemma~\ref{lem:correlated_map} (closed form under the \(\rho\)-shared model): Appendix~B.
  \item Theorem~\ref{thm:phase_transition} (phase transition for deep majority trees): Appendix~C.
  \item Theorem~\ref{thm:small_signal_band} (small-signal amplification band), Theorem~\ref{thm:opt_x_growth} (growth-regime compute allocation),
        and Corollary~\ref{cor:Bcrit} (budget threshold): Appendix~D.
  \item Theorem~\ref{thm:mixing_depth} (finite-depth convergence / mixing depth): Appendix~E.
  \item Proposition~\ref{prop:centralization} and Proposition~\ref{prop:chain_decay} (topology baselines): Appendix~F.
  \item Theorem~\ref{thm:s_upper_half} (universal upper bound \(s\le \tfrac12\)): Appendix~H.
  \item Theorem~\ref{thm:monotone_m} and Algorithm~\ref{alg:envelope} (envelope construction): Appendix~I.
\end{itemize}
Appendix~G collects derivations for the continuous warm-up, and Appendix~\ref{app:calibration} provides experimental templates for estimating the effective parameters.

\section{Properties of Majority Maps}
\label{app:majority}

Throughout, let \(b\ge 3\) be an odd integer and write \(b=2k+1\) with \(k\in\mathbb{Z}_+\).
For \(u\in[-1,1]\), define \(p=(1+u)/2\in[0,1]\) and recall the majority map
\[
f_b(u)=2\,\Pr\!\Big(\mathrm{Bin}(b,p)\ge k+1\Big)-1.
\]

\subsection{A derivative identity for binomial tails}
\label{app:binom_tail_derivative}

We start with a calculus identity that makes monotonicity and concavity transparent.

\begin{lemma}[Derivative of the symmetric binomial tail]
\label{lem:binom_tail_derivative}
Let \(b=2k+1\) and define
\[
F(p):=\Pr\!\big(\mathrm{Bin}(b,p)\ge k+1\big)
=\sum_{j=k+1}^{2k+1}\binom{2k+1}{j}p^{j}(1-p)^{2k+1-j}.
\]
Then \(F\) is differentiable on \((0,1)\) and
\begin{equation}
F'(p)=(2k+1)\binom{2k}{k}\,p^{k}(1-p)^{k}.
\label{eq:Fprime_closed}
\end{equation}
\end{lemma}

\begin{proof}
Differentiate term-by-term for \(p\in(0,1)\).
Using the identity \(j\binom{2k+1}{j}=(2k+1)\binom{2k}{j-1}\) and \((2k+1-j)\binom{2k+1}{j}=(2k+1)\binom{2k}{j}\), we obtain
\begin{align*}
\frac{d}{dp}\Big[\binom{2k+1}{j}p^{j}(1-p)^{2k+1-j}\Big]
&=\binom{2k+1}{j}\Big(jp^{j-1}(1-p)^{2k+1-j}-(2k+1-j)p^{j}(1-p)^{2k-j}\Big)\\
&=(2k+1)\Big[\binom{2k}{j-1}p^{j-1}(1-p)^{2k+1-j}-\binom{2k}{j}p^{j}(1-p)^{2k-j}\Big].
\end{align*}
Summing over \(j=k+1,\dots,2k+1\) telescopes: the negative term at index \(j\) cancels the positive term at index \(j+1\).
Only the first positive term (at \(j=k+1\)) remains, yielding
\[
F'(p)=(2k+1)\binom{2k}{k}p^{k}(1-p)^{k}.
\]
\end{proof}

\subsection{Proof of Lemma~\ref{lem:fb_props}}
\label{app:fb_props_proof}

\begin{proof}[Proof of Lemma~\ref{lem:fb_props}]
We prove each item.

\paragraph{(1) Oddness, monotonicity, and endpoints.}
Let \(p=(1+u)/2\).
Flipping the sign of the bias corresponds to sending \(p\mapsto 1-p\), i.e., swapping ``correct'' and ``incorrect'' votes.
By symmetry of the binomial distribution,
\[
\Pr\!\big(\mathrm{Bin}(b,1-p)\ge k+1\big)=\Pr\!\big(\mathrm{Bin}(b,p)\le k\big)=1-\Pr\!\big(\mathrm{Bin}(b,p)\ge k+1\big),
\]
so \(f_b(-u)=2(1-F(p))-1=-f_b(u)\), hence \(f_b\) is odd.
Also \(f_b(0)=0\) and \(f_b(1)=1\) are immediate from the definition.

To see that \(f_b\) is increasing on \([0,1]\), note that \(u\mapsto p=(1+u)/2\) is increasing and, for fixed \(b\) and threshold \(k+1\), the binomial tail \(p\mapsto F(p)\) is increasing.
Formally, Lemma~\ref{lem:binom_tail_derivative} gives \(F'(p)\ge 0\) on \((0,1)\), hence \(F\) is nondecreasing and so is \(f_b\).

\paragraph{(2) Concavity on \([0,1]\).}
Differentiate \(f_b(u)=2F((1+u)/2)-1\).
Since \(dp/du=1/2\), we have
\begin{equation}
f_b'(u)=F'\!\Big(\frac{1+u}{2}\Big).
\label{eq:fbprime_Fprime}
\end{equation}
Plugging \eqref{eq:Fprime_closed} into \eqref{eq:fbprime_Fprime} yields, for \(u\in(-1,1)\),
\begin{equation}
f_b'(u)
=(2k+1)\binom{2k}{k}\Big(\frac{1+u}{2}\Big)^{k}\Big(\frac{1-u}{2}\Big)^{k}
=
\frac{(2k+1)\binom{2k}{k}}{2^{2k}}\,(1-u^2)^{k}.
\label{eq:fbprime_closed_u}
\end{equation}
The right-hand side is nonnegative and strictly decreasing for \(u\in(0,1)\) (because \((1-u^2)^k\) strictly decreases on \((0,1)\)).
Therefore \(f_b'\) is nonincreasing on \([0,1]\), which implies that \(f_b\) is concave on \([0,1]\) (and in fact strictly concave on \((0,1)\) for \(b\ge 3\)).

\paragraph{(3) Derivative at the origin.}
Setting \(u=0\) in \eqref{eq:fbprime_closed_u} gives
\[
f_b'(0)=\frac{(2k+1)\binom{2k}{k}}{2^{2k}}
=
\frac{b\binom{b-1}{(b-1)/2}}{2^{b-1}},
\]
which is exactly \eqref{eq:fb_prime0}.
\end{proof}

\paragraph{A useful corollary: majority improves independent votes.}
Concavity on \([0,1]\) with \(f_b(0)=0\) and \(f_b(1)=1\) implies \(f_b(u)\ge u\) for all \(u\in[0,1]\), with strict inequality for \(u\in(0,1)\) because \(f_b\) is strictly concave (for \(b\ge 3\)).
We use this fact in Appendix~C when treating the \(\gamma=1\) case.

\section{Correlated Aggregation via the \texorpdfstring{$\rho$}{rho}-Shared Model}
\label{app:correlated}

Recall Definition~\ref{def:rho_shared}.
Let \(S_i:=V_iY\in\{-1,+1\}\) denote signed correctness.

\subsection{Relationship to pairwise correlation}
\label{app:rho_pairwise}

\begin{lemma}[Pairwise correlation equals \(\rho\)]
\label{lem:rho_pairwise}
Under Definition~\ref{def:rho_shared}, for any \(i\neq j\),
\[
\mathrm{Corr}(S_i,S_j)=\rho.
\]
\end{lemma}

\begin{proof}
By construction, \(\mathbb{E}[S_i]=u\) for every \(i\), so \(\mathrm{Var}(S_i)=1-u^2\).
Moreover,
\[
\mathbb{E}[S_iS_j]
=
\rho\cdot 1 + (1-\rho)\cdot \mathbb{E}[S_i]\mathbb{E}[S_j]
=
\rho + (1-\rho)u^2.
\]
Thus \(\mathrm{Cov}(S_i,S_j)=\mathbb{E}[S_iS_j]-u^2=\rho(1-u^2)\), and dividing by \(\sqrt{\mathrm{Var}(S_i)\mathrm{Var}(S_j)}=1-u^2\) yields \(\mathrm{Corr}(S_i,S_j)=\rho\).
\end{proof}

\subsection{Proof of Lemma~\ref{lem:correlated_map}}
\label{app:correlated_map_proof}

\begin{proof}[Proof of Lemma~\ref{lem:correlated_map}]
Condition on the mixture event in Definition~\ref{def:rho_shared}.
With probability \(\rho\), all child votes equal a common \(V\) with \(\mathbb{E}[VY]=u\), and majority returns \(V\), hence output bias \(u\).
With probability \(1-\rho\), the votes are i.i.d.\ with bias \(u\), so the output bias is \(f_b(u)\) by definition.
Taking expectation over the mixture yields
\[
f_{b,\rho}(u)=\rho\,u+(1-\rho)\,f_b(u),
\]
which is \eqref{eq:fbrho_def}.
Since \(f_b\) is increasing and concave on \([0,1]\) (Lemma~\ref{lem:fb_props}), the mixture \(f_{b,\rho}\) is also increasing and concave on \([0,1]\).
Differentiating at the origin gives \eqref{eq:fbrho_prime0}.
\end{proof}

\section{Binary Phase Transition Proofs}
\label{app:phase_transition}

We prove Theorem~\ref{thm:phase_transition}.
Throughout, fix odd \(b\ge 3\), \(\rho\in[0,1)\), and \(\gamma\in(0,1]\), and define
\[
T(u)=f_{b,\rho}(\gamma u)=\rho\gamma u+(1-\rho)f_b(\gamma u),\qquad u\in[0,1].
\]
By Lemma~\ref{lem:fb_props} and Lemma~\ref{lem:correlated_map}, \(T\) is continuous, increasing, and concave on \([0,1]\), with \(T(0)=0\).
Because \(1-\rho>0\) and \(f_b\) is strictly concave on \((0,1)\), the map \(T\) is strictly concave on \((0,1)\).

\subsection{A ratio monotonicity lemma}
\label{app:ratio_monotone}

\begin{lemma}[Concavity implies decreasing ratio]
\label{lem:ratio_monotone}
Let \(T:[0,1]\to\mathbb{R}_+\) be concave with \(T(0)=0\).
Then the ratio \(r(u):=T(u)/u\) is nonincreasing on \((0,1]\).
\end{lemma}

\begin{proof}
Fix \(0<u<v\le 1\).
Concavity gives
\[
T(u)=T\Big(\frac{u}{v}v+\Big(1-\frac{u}{v}\Big)0\Big)
\ge \frac{u}{v}T(v)+\Big(1-\frac{u}{v}\Big)T(0)
=\frac{u}{v}T(v).
\]
Dividing by \(u>0\) yields \(T(u)/u\ge T(v)/v\).
\end{proof}

\subsection{Proof of Theorem~\ref{thm:phase_transition}}
\label{app:phase_transition_proof}

\begin{proof}[Proof of Theorem~\ref{thm:phase_transition}]
Let \(\alpha_\rho=T'(0)=\gamma f_{b,\rho}'(0)\) as in \eqref{eq:alpha_rho_def}.
Define \(h(u):=T(u)-u\).
Since \(T\) is concave, \(h\) is concave (and in fact strictly concave on \((0,1)\)) with \(h(0)=0\) and \(h'(0)=\alpha_\rho-1\).

\paragraph{(1) Subcritical collapse: \(\alpha_\rho\le 1\).}
If \(\alpha_\rho\le 1\), then \(h'(0)\le 0\).
Because \(h\) is concave and \(h(0)=0\), for any \(u\in[0,1]\) we have the supporting-line bound
\[
h(u)\le h(0)+h'(0)u \le 0.
\]
Thus \(T(u)\le u\) on \([0,1]\).
The recursion \(\mu_{t+1}=T(\mu_t)\) is therefore monotone nonincreasing and bounded below by 0, hence it converges to some \(\bar{\mu}\in[0,1]\).
By continuity of \(T\), the limit is a fixed point: \(\bar{\mu}=T(\bar{\mu})\), i.e., \(h(\bar{\mu})=0\).

We now show \(\bar{\mu}=0\).
If \(\bar{\mu}>0\), then \(h(\bar{\mu})=0\) and \(h(u)\le 0\) for all \(u\in[0,1]\).
Concavity then forces \(h(u)=0\) for all \(u\in[0,\bar{\mu}]\): indeed, for any \(0<u<\bar{\mu}\),
\[
h(u)\ge \frac{u}{\bar{\mu}}h(\bar{\mu})+\Big(1-\frac{u}{\bar{\mu}}\Big)h(0)=0,
\]
while also \(h(u)\le 0\), hence \(h(u)=0\).
This contradicts strict concavity of \(h\) on \((0,1)\) unless \(\bar{\mu}=0\).
Therefore \(\mu_t\to 0\).

\paragraph{(2) Supercritical amplification: \(\alpha_\rho>1\).}
Assume \(\alpha_\rho>1\), so \(h'(0)>0\) and hence \(h(u)>0\) for all sufficiently small \(u>0\).
Next, we show that \(h(1)\le 0\), with strict inequality when \(\gamma<1\).
Indeed,
\[
h(1)=T(1)-1=f_{b,\rho}(\gamma)-1.
\]
If \(\gamma=1\), then \(f_{b,\rho}(1)=1\) so \(h(1)=0\).
If \(\gamma<1\), then \(0\le \gamma<1\) and \(f_b(\gamma)<1\) (majority cannot be perfectly correct when each vote has bias \(<1\)), hence
\(f_{b,\rho}(\gamma)=\rho\gamma+(1-\rho)f_b(\gamma)<\rho+(1-\rho)\cdot 1=1\), so \(h(1)<0\).

Since \(h\) is continuous, \(h(0)=0\), \(h(u)>0\) for small \(u>0\), and \(h(1)\le 0\), there exists at least one root \(\mu^\star\in(0,1]\) with \(h(\mu^\star)=0\), i.e., \(T(\mu^\star)=\mu^\star\).
Uniqueness follows from strict concavity: if there were two distinct roots in \((0,1]\), then \(h\) would have to be linear on the interval between them, contradicting strict concavity.
Thus the fixed point \(\mu^\star\in(0,1]\) is unique.

We now show global convergence to \(\mu^\star\) for any \(\mu_0\in(0,1]\).
By Lemma~\ref{lem:ratio_monotone}, the ratio \(r(u)=T(u)/u\) is nonincreasing on \((0,1]\).
Because \(r(0^+)=\alpha_\rho>1\) and \(r(\mu^\star)=1\), we have \(r(u)>1\) for \(u\in(0,\mu^\star)\) and \(r(u)<1\) for \(u\in(\mu^\star,1]\).
Equivalently,
\[
u\in(0,\mu^\star)\Rightarrow T(u)>u,\qquad
u\in(\mu^\star,1]\Rightarrow T(u)<u.
\]
If \(\mu_0\in(0,\mu^\star)\), then \(\{\mu_t\}\) is strictly increasing and bounded above by \(\mu^\star\), hence converges to some \(\bar{\mu}\le\mu^\star\); by continuity, \(\bar{\mu}\) is a fixed point, so \(\bar{\mu}=\mu^\star\).
If \(\mu_0\in(\mu^\star,1]\), then \(\{\mu_t\}\) is strictly decreasing and bounded below by \(\mu^\star\), hence converges to a fixed point \(\bar{\mu}\ge \mu^\star\), again forcing \(\bar{\mu}=\mu^\star\).
This proves \(\mu_t\to\mu^\star\) for all \(\mu_0\in(0,1]\).

Finally, the statements about \(\mu^\star\) follow from the above boundary cases:
if \(\gamma=1\), then \(h(1)=0\) and the unique positive fixed point is \(\mu^\star=1\);
if \(\gamma<1\), then \(h(1)<0\) so \(\mu^\star<1\).
\end{proof}

\paragraph{A stability fact used later.}
At the attracting fixed point \(\mu^\star\), we necessarily have \(T'(\mu^\star)<1\).
Indeed, since \(h(\mu^\star)=0\) and \(h\) is concave with \(h(u)>0\) for \(u<\mu^\star\), the tangent line at \(\mu^\star\) must have strictly negative slope: \(h'(\mu^\star)<0\), i.e., \(T'(\mu^\star)<1\).
We use this to obtain geometric convergence in Appendix~E.

\section{Small-Signal Bounds and Budget Thresholds}
\label{app:small_signal_budget}

This section proves Theorem~\ref{thm:small_signal_band}, Theorem~\ref{thm:opt_x_growth}, and Corollary~\ref{cor:Bcrit}.

\subsection{Linear-band bounds and proof of Theorem~\ref{thm:small_signal_band}}
\label{app:small_signal_band_proof}

Recall \(T(u)=f_{b,\rho}(\gamma u)\), which is concave on \([0,1]\) with \(T(0)=0\) and derivative \(\alpha_\rho=T'(0)\).

\begin{lemma}[Tangent upper bound at the origin]
\label{lem:tangent_upper}
If \(T\) is concave on \([0,1]\) and differentiable at \(0\), then
\[
T(u)\le T(0)+T'(0)\,u=\alpha_\rho\,u,\qquad u\in[0,1].
\]
\end{lemma}

\begin{proof}
For concave functions, every tangent line is a global upper bound.
In particular, differentiability at \(0\) implies that the affine function \(u\mapsto T(0)+T'(0)u\) supports \(T\) at \(0\), hence upper-bounds \(T\) on \([0,1]\).
\end{proof}

\begin{proof}[Proof of Theorem~\ref{thm:small_signal_band}]
Fix \(\eta\in(0,1)\).
The upper inequality in \eqref{eq:linear_band} is Lemma~\ref{lem:tangent_upper}.
For the lower inequality, note that \(T\) is differentiable at \(0\), so
\[
\lim_{u\downarrow 0}\frac{T(u)}{u}=T'(0)=\alpha_\rho.
\]
Therefore, there exists \(\delta\in(0,1]\) such that for all \(u\in(0,\delta]\),
\[
\frac{T(u)}{u}\ge (1-\eta)\alpha_\rho,
\]
which is exactly the lower bound in \eqref{eq:linear_band}.

To obtain the iterative sandwich \eqref{eq:mu_sandwich}, assume \(\mu_0\alpha_\rho^{L}\le \delta\).
By Lemma~\ref{lem:tangent_upper} and induction, \(\mu_t\le \mu_0\alpha_\rho^{t}\le \mu_0\alpha_\rho^{L}\le \delta\) for all \(t\le L\).
Thus the lower bound in \eqref{eq:linear_band} applies at each step:
\(\mu_{t+1}=T(\mu_t)\ge (1-\eta)\alpha_\rho\,\mu_t\), so again by induction
\(\mu_t\ge (1-\eta)^t\mu_0\alpha_\rho^t\).
Setting \(t=L\) gives \eqref{eq:mu_sandwich}.
\end{proof}

\subsection{Closed-form compute allocation and proof of Theorem~\ref{thm:opt_x_growth}}
\label{app:xstar_growth_proof}

We treat the growth-regime surrogate objective in \eqref{eq:grow_objective} as an exact function of \(x>0\) for fixed \((B,b,m)\), ignoring multiplicative constants independent of \(x\):
\[
\mu_{\mathrm{grow}}(x)\propto x^\beta(x+c_0)^{-s},
\qquad c_0=c_0(b,m)>0.
\]
Taking logs gives
\[
\phi(x):=\log \mu_{\mathrm{grow}}(x)=\beta\log x - s\log(x+c_0) + \text{const}(B,b,m).
\]

\begin{proof}[Proof of Theorem~\ref{thm:opt_x_growth}]
Differentiate \(\phi\):
\[
\phi'(x)=\frac{\beta}{x}-\frac{s}{x+c_0}
=\frac{\beta(x+c_0)-sx}{x(x+c_0)}
=\frac{\beta c_0-(s-\beta)x}{x(x+c_0)}.
\]
If \(s>\beta\), the unique critical point is at
\(x^\star=\frac{\beta}{s-\beta}c_0\), which is \eqref{eq:xstar_growth}.
A second derivative check shows it is a maximizer:
\[
\phi''(x)=-\frac{\beta}{x^2}+\frac{s}{(x+c_0)^2}<0\qquad\text{at }x=x^\star,
\]
because at \(x^\star\) we have \(\frac{s}{(x^\star+c_0)^2}=\frac{s(s-\beta)^2}{s^2 c_0^2}<\frac{\beta(s-\beta)^2}{\beta^2 c_0^2}=\frac{\beta}{(x^\star)^2}\).
Equivalently, \(\phi\) is strictly concave around \(x^\star\).

Substituting \(x^\star=\frac{\beta}{s-\beta}c_0\) into \eqref{eq:grow_objective} and simplifying yields \eqref{eq:mu_grow_star}.
Indeed,
\[
\frac{x^\star}{x^\star+c_0}=\frac{\beta}{s},
\qquad
x^\star+c_0=\frac{s}{s-\beta}c_0,
\]
so
\[
x^{\star\beta}\left(\frac{1}{x^\star+c_0}\right)^{s}
=
\left(\frac{\beta}{s-\beta}c_0\right)^{\beta}\left(\frac{s-\beta}{s}\frac{1}{c_0}\right)^{s}
=
\beta^\beta (s-\beta)^{s-\beta}s^{-s}\,c_0^{\beta-s}
=
\kappa(s,\beta)\,c_0^{\beta-s}.
\]

If \(s\le \beta\), then \(\phi'(x)\ge 0\) for all \(x>0\), so the growth surrogate is nondecreasing in \(x\).
Under the budget constraint \(B\approx N(x+c_0)\) with \(N\ge 1\), the maximum feasible \(x\) is on the order of \(B-c_0\) (corresponding to \(N=1\)), i.e., ``scale up rather than scale out.''
\end{proof}

\subsection{Proof of Corollary~\ref{cor:Bcrit}}
\label{app:Bcrit_proof}

\begin{proof}[Proof of Corollary~\ref{cor:Bcrit}]
In the growth surrogate, Theorem~\ref{thm:opt_x_growth} gives
\[
\mu_{\mathrm{grow}}^\star(B;b,m)\approx k\,\kappa(s,\beta)\,c_0^{\beta-s}\,B^{s}.
\]
The single-agent surrogate is \(\mu_{\mathrm{single}}(B)\approx kB^\beta\).
Requiring \(\mu_{\mathrm{grow}}^\star(B;b,m)\ge \mu_{\mathrm{single}}(B)\) and canceling \(k>0\) yields
\[
\kappa(s,\beta)\,c_0^{\beta-s}\,B^{s-\beta}\ge 1
\qquad\Longleftrightarrow\qquad
B^{s-\beta}\ge \kappa(s,\beta)^{-1}c_0^{s-\beta}.
\]
Since \(s>\beta\), raising both sides to the power \(1/(s-\beta)\) gives the sufficient condition
\[
B\ge c_0\,\kappa(s,\beta)^{-\frac{1}{s-\beta}}=B_{\mathrm{crit}}(b,m),
\]
which is \eqref{eq:Bcrit}.
\end{proof}

\subsection{A remark on the \(s\le \beta\) regime}
\label{app:s_le_beta_remark}

When \(s\le \beta\), the growth surrogate cannot deliver an exponent advantage over the single-agent curve.
More precisely, optimizing \eqref{eq:grow_objective} over \(x\) yields a prediction that scales at most as \(B^\beta\) (and typically prefers \(N\approx 1\)).
This is the sense in which the condition \(s>\beta\) is necessary for scale-out \emph{exponent} gains in the growth regime.

\section{Saturation and Mixing Depth}
\label{app:mixing}

We prove Theorem~\ref{thm:mixing_depth}.
Assume the supercritical regime \(\alpha_\rho>1\), so by Theorem~\ref{thm:phase_transition} there is a unique attracting fixed point \(\mu^\star\in(0,1]\) and \(\mu_t\to \mu^\star\) for every \(\mu_0>0\).

\subsection{Two-phase mixing analysis and proof of Theorem~\ref{thm:mixing_depth}}
\label{app:mixing_depth_proof}

\begin{proof}[Proof of Theorem~\ref{thm:mixing_depth}]
If \(\mu_0\ge \mu^\star\), then \(\mu_t\ge \mu^\star\) for all \(t\ge 0\) (the recursion is monotone decreasing toward \(\mu^\star\) from above), and therefore \(\mu_L\ge (1-\varepsilon)\mu^\star\) holds for every \(L\) with \(L_{\mathrm{mix}}(\varepsilon)=0\).
We therefore focus on the case \(\mu_0\in(0,\mu^\star)\), where \(\mu_t\uparrow \mu^\star\).

\paragraph{Phase I: reaching a constant fraction of \(\mu^\star\).}
Define \(r(u)=T(u)/u\) for \(u>0\), which is continuous on \((0,1]\) and nonincreasing by Lemma~\ref{lem:ratio_monotone}.
Because \(r(\mu^\star)=1\) and \(r(u)>1\) for \(u\in(0,\mu^\star)\), pick any reference level \(u_1\in(0,\mu^\star)\) (e.g., \(u_1=\mu^\star/2\)) and set
\[
\underline{\alpha}:=r(u_1)=\frac{T(u_1)}{u_1}>1.
\]
For any \(u\in(0,u_1]\), monotonicity of \(r\) gives \(T(u)/u=r(u)\ge r(u_1)=\underline{\alpha}\), hence \(T(u)\ge \underline{\alpha}u\).
As long as \(\mu_t\le u_1\), we therefore have \(\mu_{t+1}\ge \underline{\alpha}\mu_t\), and by induction
\[
\mu_t\ge \mu_0\,\underline{\alpha}^{\,t}\qquad\text{until the first time }\mu_t>u_1.
\]
Thus the hitting time
\[
t_1:=\min\{t:\ \mu_t\ge u_1\}
\]
satisfies
\[
t_1\le \left\lceil \frac{\log(u_1/\mu_0)}{\log \underline{\alpha}}\right\rceil
=O\!\left(\log\frac{1}{\mu_0}\right),
\]
with constants depending only on \(u_1\) (hence on \(b,\rho,\gamma\)).

\paragraph{Phase II: geometric contraction near \(\mu^\star\).}
As noted at the end of Appendix~C, we have \(T'(\mu^\star)<1\).
By continuity of \(T'\) and compactness, there exists \(u_0\in(u_1,\mu^\star)\) and a constant \(\kappa\in(0,1)\) such that
\begin{equation}
\sup_{u\in[u_0,\mu^\star]}T'(u)\le \kappa<1.
\label{eq:kappa_contraction}
\end{equation}
Once \(\mu_t\ge u_0\), define the gap \(d_t:=\mu^\star-\mu_t\in[0,\mu^\star-u_0]\).
By the mean value theorem, for each \(t\) with \(\mu_t\in[u_0,\mu^\star]\) there exists \(\xi_t\in[\mu_t,\mu^\star]\subseteq[u_0,\mu^\star]\) such that
\[
d_{t+1}
=
\mu^\star-\mu_{t+1}
=
T(\mu^\star)-T(\mu_t)
=
T'(\xi_t)\,(\mu^\star-\mu_t)
\le \kappa\,d_t.
\]
Iterating yields \(d_{t+n}\le \kappa^n d_t\).
In particular, once we have reached \(u_0\), to ensure \(d_{t+n}\le \varepsilon \mu^\star\) it suffices that
\[
n\ge \left\lceil \frac{\log(d_t/(\varepsilon\mu^\star))}{\log(1/\kappa)}\right\rceil
=O\!\left(\log\frac{1}{\varepsilon}\right),
\]
where the constant depends on \(\kappa\) and the initial gap bound \(d_t\le \mu^\star\).

\paragraph{Combining phases.}
Since \(\mu_t\uparrow\mu^\star\), the time to go from \(u_1\) to \(u_0\) is finite and depends only on \((b,\rho,\gamma)\); absorb it into constants.
Combining the bounds gives a depth
\[
L_{\mathrm{mix}}(\varepsilon)
\le
C_1\log\frac{1}{\mu_0}+C_2\log\frac{1}{\varepsilon}+C_3,
\]
with constants depending only on \((b,\rho,\gamma)\), such that \(\mu_L\ge (1-\varepsilon)\mu^\star\) for all \(L\ge L_{\mathrm{mix}}(\varepsilon)\).
This establishes the claimed scaling.
\end{proof}

\subsection{Clipped objectives and approximation guarantees}
\label{app:clipped_guarantees}

The clipped proxy \(\widehat{\mu}_L=\min\{\mu^\star,\mu_0\alpha_\rho^L\}\) in \eqref{eq:clipped_mu} is most naturally interpreted for the regime \(\mu_0\le \mu^\star\), which includes the small-signal settings used in Section~\ref{sec:budget_thresholds}.
In that case, monotonicity and Theorem~\ref{thm:phase_transition} imply \(\mu_L\le \mu^\star\) for all \(L\), while Lemma~\ref{lem:tangent_upper} implies \(\mu_L\le \mu_0\alpha_\rho^L\).
Therefore \(\mu_L\le \widehat{\mu}_L\), justifying the global upper bound claim in Section~\ref{sec:saturation_binary}.
Theorems~\ref{thm:small_signal_band} and~\ref{thm:mixing_depth} provide matching lower bounds in the growth and saturation regimes, respectively.

\section{Topology Comparison Proofs}
\label{app:topology}

\subsection{Star upper bounds under context constraints}
\label{app:star_context}

A star node that reads \(m\) tokens per leaf can aggregate at most
\(N_{\max}(m)=\lfloor W/m\rfloor\) leaves within a context budget \(W\).
This is the sense in which a star saturates under a finite context window: beyond \(N_{\max}(m)\), additional leaves cannot be included without increasing \(m\) (which changes both cost and fidelity).

\subsection{Proof of Proposition~\ref{prop:centralization}}
\label{app:centralization_proof}

\begin{proof}[Proof of Proposition~\ref{prop:centralization}]
Any predictor \(\widehat{Y}=\phi(Z)\) induces a predictor on \(X\) via composition \(\psi(X):=\phi(Z(X,U))\), where \(U\) denotes any protocol randomness (assumed independent of \(Y\)).
Therefore, the set of predictors measurable with respect to \(Z\) is a subset of those measurable with respect to \(X\), implying the optimal risk using \(X\) is no larger than the optimal risk using \(Z\).
\end{proof}

\subsection{Proof of Proposition~\ref{prop:chain_decay}}
\label{app:chain_decay_proof}

\begin{proof}[Proof of Proposition~\ref{prop:chain_decay}]
In the channel model of Section~\ref{sec:communication}, each hop multiplies the incoming bias by \(\gamma(m)\).
After \(L\) independent hops, the bias is \(\mu_0\gamma(m)^L\).
\end{proof}

\subsection{Tree vs.\ star dominance beyond a threshold}
\label{app:tree_vs_star}

The comparison between a context-limited star and a supercritical tree in Section~\ref{sec:phase_diagrams} is driven by two facts:
(i) the star cannot use more than \(N_{\max}(m)\) leaves for fixed \(m\), hence eventually must spend additional budget on scale-up (which grows like \(B^\beta\) in the small-signal regime),
while (ii) a tree with exponent \(s>\beta\) admits a growth regime with \(B^s\) scaling (Appendix~D) before saturating at \(\mu^\star\) (Appendix~E).
This yields a provable intermediate budget range where trees can dominate under context constraints, formalized through the sufficient condition \(B\ge B_{\mathrm{crit}}\) in Corollary~\ref{cor:Bcrit}.

\section{Continuous Extension Proofs}
\label{app:continuous}

This appendix derives the closed-form mean-squared error (MSE) recursions used in the continuous warm-up (Section~\ref{sec:continuous_warmup}).
Throughout, let \(Y\in[0,1]\) be the ground truth and suppose each leaf outputs \(\widehat{Y}_i=Y+E_i\) with \(\mathbb{E}[E_i]=0\), \(\mathbb{E}[E_i^2]=v_0\), and \(\mathrm{Corr}(E_i,E_j)=\rho\) for \(i\neq j\).

\subsection{Star, chain, and tree recursions}
\label{app:continuous_recursions}

\paragraph{Star.}
In the star topology, each leaf transmits \(\widehat{Y}_i\) over the additive channel of Section~\ref{sec:continuous_warmup}, so the aggregator receives \(\widehat{Y}_i+\eta_i\) with \(\mathbb{E}[\eta_i]=0\) and \(\mathbb{E}[\eta_i^2]=\sigma_c^2(m)\), independent across leaves and independent of the leaf errors \(E_i\).
The aggregator outputs the average of the received messages,
\[
\widehat{Y}_{\mathrm{star}}=\frac{1}{N}\sum_{i=1}^N (\widehat{Y}_i+\eta_i)
=Y+\bar{E}+\bar{\eta},
\qquad
\bar{E}:=\frac{1}{N}\sum_{i=1}^N E_i,\;\;
\bar{\eta}:=\frac{1}{N}\sum_{i=1}^N \eta_i.
\]
Because \(\bar{E}\) and \(\bar{\eta}\) are independent and mean-zero, the MSE is
\[
\mathbb{E}[(\bar{E}+\bar{\eta})^2]=\mathbb{E}[\bar{E}^2]+\mathbb{E}[\bar{\eta}^2]
=v_0\Big(\rho+\frac{1-\rho}{N}\Big)+\frac{\sigma_c^2(m)}{N},
\]
which is \eqref{eq:star_mse}.

\paragraph{Chain.}
In the chain model of Section~\ref{sec:continuous_warmup}, each hop adds independent communication noise \(C_\ell\) with \(\mathbb{E}[C_\ell]=0\), \(\mathbb{E}[C_\ell^2]=\sigma_c^2(m)\), so after \(L\) hops the error variance is \(v_0+L\sigma_c^2(m)\), giving \eqref{eq:chain_mse}.

\paragraph{Tree.}
For a \(b\)-ary tree, each internal node averages the \(b\) \emph{received} child messages.
Under the edge channel model of Section~\ref{sec:continuous_warmup}, child \(j\) transmits \(\widehat{Y}_j\) and the parent receives \(\widehat{Y}_j+\eta_j\), where \(\mathbb{E}[\eta_j]=0\) and \(\mathbb{E}[\eta_j^2]=\sigma_c^2(m)\), independent across edges and independent of estimation errors.
If the child errors have variance \(v_t\) and pairwise correlation \(\rho\), then averaging the \(b\) child errors yields variance \(v_t(\rho+(1-\rho)/b)\) (the same computation as the star with \(N=b\)).
Averaging the \(b\) independent channel noises adds variance \(\sigma_c^2(m)/b\).
Thus the recursion is
\[
v_{t+1}=v_t\Big(\rho+\frac{1-\rho}{b}\Big)+\frac{\sigma_c^2(m)}{b},
\]
which is \eqref{eq:tree_rec}.
Solving this affine recursion gives \eqref{eq:tree_closed}, and taking \(t\to\infty\) yields the floor \eqref{eq:tree_floor}.
\subsection{Mixing depth in the continuous recursion}
\label{app:continuous_mixing}

Let \(\lambda:=\rho+(1-\rho)/b\in(0,1)\).
From the closed form \eqref{eq:tree_closed},
\[
v_t-v^\star=\lambda^t(v_0-v^\star),
\]
so reaching relative error \(v_t\le v^\star+\varepsilon(v_0-v^\star)\) is equivalent to \(\lambda^t\le \varepsilon\), i.e.,
\(t\ge \frac{\log(1/\varepsilon)}{\log(1/\lambda)}\), which is \eqref{eq:cont_mix_depth}.
This is the continuous analogue of the binary mixing-depth statement in Theorem~\ref{thm:mixing_depth}.

\section{Universal Upper Bound on Amplification Exponents}
\label{app:universal_bound}

We prove Theorem~\ref{thm:s_upper_half}.

\subsection{A standard bound on the central binomial coefficient}
\label{app:central_binom_bound}

The proof relies on the classical inequality
\begin{equation}
\binom{2k}{k}\le \frac{4^k}{\sqrt{\pi k}},\qquad k\ge 1,
\label{eq:central_binom_stirling}
\end{equation}
which we derive here from an integral bound (avoiding asymptotic approximations).

\begin{lemma}[Central binomial coefficient bound]
\label{lem:central_binom}
For every integer \(k\ge 1\),
\(\binom{2k}{k}\le \frac{4^k}{\sqrt{\pi k}}\).
\end{lemma}

\begin{proof}
Let \(I_k:=\int_{0}^{\pi/2}\sin^{2k}\theta\,d\theta\).
A standard integration-by-parts recursion gives \(I_k=\frac{2k-1}{2k}I_{k-1}\) with \(I_0=\pi/2\), hence
\[
I_k=\frac{\pi}{2}\cdot \frac{(2k-1)!!}{(2k)!!}
=\frac{\pi}{2}\cdot \frac{(2k)!}{2^{2k}(k!)^2}
=\frac{\pi}{2}\cdot \frac{1}{4^k}\binom{2k}{k}.
\]
Equivalently,
\begin{equation}
\binom{2k}{k}=\frac{2\cdot 4^k}{\pi}\,I_k.
\label{eq:binom_integral_identity}
\end{equation}

We next upper bound \(I_k\).
Change variables \(x=\frac{\pi}{2}-\theta\) to get
\(I_k=\int_{0}^{\pi/2}\cos^{2k}x\,dx\).
For \(x\in[0,\pi/2]\), define \(g(x)=\log(\cos x)+x^2/2\).
Then \(g(0)=0\) and \(g'(x)=-\tan x + x\le 0\) because \(\tan x\ge x\) for \(x\ge 0\).
Thus \(g(x)\le 0\) and
\(\cos x\le e^{-x^2/2}\) on \([0,\pi/2]\).
Therefore,
\[
I_k=\int_{0}^{\pi/2}\cos^{2k}x\,dx
\le \int_{0}^{\pi/2}e^{-k x^2}\,dx
\le \int_{0}^{\infty}e^{-k x^2}\,dx
=\frac{\sqrt{\pi}}{2\sqrt{k}}.
\]
Plugging this into \eqref{eq:binom_integral_identity} yields
\[
\binom{2k}{k}\le \frac{2\cdot 4^k}{\pi}\cdot \frac{\sqrt{\pi}}{2\sqrt{k}}
=\frac{4^k}{\sqrt{\pi k}},
\]
which is \eqref{eq:central_binom_stirling}.
\end{proof}

\subsection{Proof of Theorem~\ref{thm:s_upper_half}}
\label{app:s_upper_half_proof}

\begin{proof}[Proof of Theorem~\ref{thm:s_upper_half}]
Let \(b=2k+1\) with \(k\ge 1\).
By Lemma~\ref{lem:fb_props},
\[
f_b'(0)=\frac{(2k+1)\binom{2k}{k}}{2^{2k}}=(2k+1)\binom{2k}{k}\,4^{-k}.
\]
Applying Lemma~\ref{lem:central_binom} yields
\[
f_b'(0)\le (2k+1)\cdot \frac{4^k}{\sqrt{\pi k}}\cdot 4^{-k}
=\frac{2k+1}{\sqrt{\pi k}}.
\]
Since \(\pi>3\) and \(k\ge 1\), we have \(\frac{2k+1}{\sqrt{\pi k}}\le \sqrt{2k+1}=\sqrt{b}\), proving \eqref{eq:fbprime_sqrtb}.

For the second claim, note that \(\gamma(m)\le 1\) and \(\rho+(1-\rho)f_b'(0)\le \rho+(1-\rho)\sqrt{b}\le \sqrt{b}\) because \(\sqrt{b}\ge 1\).
Therefore \(\alpha_\rho=\gamma(m)\big(\rho+(1-\rho)f_b'(0)\big)\le \sqrt{b}\), and \(s=\frac{\log\alpha_\rho}{\log b}\le \frac12\).
\end{proof}

\section{Monotone Design Curves and Envelope Algorithm}
\label{app:monotone_design}

This appendix proves Theorem~\ref{thm:monotone_m} and justifies Algorithm~\ref{alg:envelope} in Section~\ref{sec:envelope_algorithm}.

\subsection{Increasing differences and proof of Theorem~\ref{thm:monotone_m}}
\label{app:monotone_m_proof}

Fix fan-in \(b\) and consider feasible \(m\in\mathcal{M}_b=\{m\in\mathbb{Z}_+:\ bm\le W\}\).
For each feasible \(m\) that passes the screening conditions \(\alpha_\rho(b,m)>1\) and \(s_{b,m}>\beta\), define
\[
\ell_m(t):=\log \underline{\mu}_b(m;e^t)=s_{b,m}\,t+q_b(m),
\qquad t=\log B,
\]
as in \eqref{eq:line_def}--\eqref{eq:mu_bm}.
By assumption, \(\gamma(m)\) is strictly increasing in \(m\), hence so is \(\alpha_\rho(b,m)\), and therefore \(s_{b,m}=\log \alpha_\rho(b,m)/\log b\) is strictly increasing in \(m\).

\begin{lemma}[Single-crossing / increasing differences]
\label{lem:increasing_differences}
For any \(m_2>m_1\), the difference \(\ell_{m_2}(t)-\ell_{m_1}(t)\) is strictly increasing in \(t\).
\end{lemma}

\begin{proof}
Compute
\[
\ell_{m_2}(t)-\ell_{m_1}(t)=(s_{b,m_2}-s_{b,m_1})\,t+\big(q_b(m_2)-q_b(m_1)\big).
\]
Because \(s_{b,m_2}>s_{b,m_1}\), the coefficient of \(t\) is positive, so the difference is strictly increasing in \(t\).
\end{proof}

\begin{proof}[Proof of Theorem~\ref{thm:monotone_m}]
Let \(t=\log B\) and define the argmax correspondence
\[
\mathcal{A}(t):=\arg\max_{m\in\mathcal{M}_b}\ \ell_m(t).
\]
By finiteness of \(\mathcal{M}_b\), \(\mathcal{A}(t)\) is nonempty.
We show that the minimal selection \(m^{-}(t):=\min \mathcal{A}(t)\) is nondecreasing in \(t\), which implies the theorem.

Take \(t_2>t_1\) and let \(m_1=m^{-}(t_1)\in\mathcal{A}(t_1)\) and \(m_2=m^{-}(t_2)\in\mathcal{A}(t_2)\).
Suppose for contradiction that \(m_2<m_1\).
Optimality gives
\[
\ell_{m_1}(t_1)\ge \ell_{m_2}(t_1),
\qquad
\ell_{m_2}(t_2)\ge \ell_{m_1}(t_2).
\]
Adding these inequalities yields
\begin{equation}
\ell_{m_1}(t_1)+\ell_{m_2}(t_2)\ \ge\ \ell_{m_2}(t_1)+\ell_{m_1}(t_2).
\label{eq:swap_ineq}
\end{equation}
On the other hand, Lemma~\ref{lem:increasing_differences} with \(m_1>m_2\) and \(t_2>t_1\) implies
\[
\big(\ell_{m_1}(t_2)-\ell_{m_2}(t_2)\big)
>
\big(\ell_{m_1}(t_1)-\ell_{m_2}(t_1)\big),
\]
which rearranges to
\[
\ell_{m_1}(t_2)+\ell_{m_2}(t_1)\ >\ \ell_{m_1}(t_1)+\ell_{m_2}(t_2),
\]
contradicting \eqref{eq:swap_ineq}.
Therefore \(m_2\ge m_1\), i.e., \(m^{-}(t)\) is nondecreasing.
Translating back from \(t=\log B\) proves that an optimal message length can be chosen nondecreasing in \(B\).
\end{proof}

\subsection{Upper-envelope construction and correctness of Algorithm~\ref{alg:envelope}}
\label{app:envelope_correctness}

We sketch a standard correctness argument for the stack-based envelope construction used in Algorithm~\ref{alg:envelope}.
Consider a finite family of lines \(\ell_i(t)=s_i t+q_i\) with strictly increasing slopes \(s_1<s_2<\cdots<s_M\).
Define the upper envelope \(E(t)=\max_i \ell_i(t)\).

\begin{lemma}[Envelope structure for increasing slopes]
\label{lem:envelope_structure}
If \(s_i\) are strictly increasing, then there exist indices \(1\le i_1<i_2<\cdots<i_r\le M\) and breakpoints
\(-\infty=\tau_0<\tau_1<\cdots<\tau_r=\infty\) such that for every \(j\in\{1,\dots,r\}\),
\[
E(t)=\ell_{i_j}(t)\quad\text{for all }t\in[\tau_{j-1},\tau_j).
\]
That is, the maximizing index is nondecreasing in \(t\) and changes only at increasing intersection points.
\end{lemma}

\begin{proof}
For any pair \(i<j\), the difference \(\ell_j(t)-\ell_i(t)\) is an increasing affine function of \(t\) (because \(s_j-s_i>0\)).
Therefore, the set where \(\ell_j\) dominates \(\ell_i\) is either empty or a half-line \([t^\star,\infty)\), where \(t^\star\) is their intersection.
This single-crossing property implies the argmax index over \(\{\ell_i\}\) is nondecreasing in \(t\), and the envelope is piecewise linear with breakpoints at pairwise intersections.
\end{proof}

Algorithm~\ref{alg:envelope} is a constructive implementation of Lemma~\ref{lem:envelope_structure}.
Processing lines in increasing slope order, it maintains a stack of candidate lines \((\ell_{i_1},\dots,\ell_{i_r})\) and the associated increasing breakpoints \((\tau_1,\dots,\tau_{r-1})\).
When a new line \(\ell\) is added, if its intersection with the current top occurs to the left of the top's activation breakpoint, then the top line is never maximal and is removed.
Because each line is pushed and popped at most once, the construction runs in linear time.
The resulting stack and breakpoint list exactly encode the envelope \(E(t)\), and the associated maximizing message length is the piecewise-constant function \(m_b^\star(B)\) in Theorem~\ref{thm:monotone_m}.

\section{Calibration and evaluation templates}
\label{app:calibration}

This appendix collects optional templates for estimating the \emph{effective} parameters of the framework in a given agent stack and for running small diagnostic evaluations.
Because the main paper is theoretical and the modeling assumptions are deliberately simplified, these templates are intended primarily to support qualitative regime checks (e.g., whether \(\alpha_\rho\) is supercritical, or whether saturation is communication- versus correlation-limited), rather than to provide a faithful simulator of all real-world effects.

\subsection{Calibration protocol: estimating \texorpdfstring{$\beta$}{beta}, \texorpdfstring{$\gamma(m)$}{gamma(m)} / \texorpdfstring{$\sigma_c^2(m)$}{sigma\_c^2(m)}, and \texorpdfstring{$\rho$}{rho}}
\label{sec:calibration_protocol}

This subsection operationalizes the ``measurable parameters'' promise from Section~\ref{sec:measurable-params}.
The goal is not to fit a complicated model, but to obtain stable estimates that are accurate enough to place the system in the correct qualitative regime of the phase diagram.

\paragraph{Single-agent scaling (\(\beta\)).}
Fix a leaf-agent template (model, prompt, tools) and define a compute knob \(x\) (e.g., reasoning tokens, number of samples, number of tool calls).
Evaluate the leaf agent on calibration tasks at a grid \(x_1<\cdots<x_K\) (roughly log-spaced).

For binary tasks, estimate accuracy \(\widehat{p}(x_k)\) and convert to bias \(\widehat{\mu}(x_k)=2\widehat{p}(x_k)-1\).
Fit a power law on a \emph{small-signal window} where \(\widehat{\mu}\) is positive and not saturated, e.g.,
\[
\log \widehat{\mu}(x_k) \approx \log k + \beta\log x_k,\qquad k\in\mathcal{K},
\]
with \(\mathcal{K}\) chosen so that \(\widehat{\mu}(x_k)\in[\mu_{\min},\mu_{\max}]\) for task-appropriate constants \(\mu_{\min}>0\ll \mu_{\max}<1\).
Report \(\widehat{\beta}\) (and optionally \(\widehat{k}\)) with bootstrap confidence intervals over tasks.
For continuous tasks, estimate \(v_0(x_k)=\mathbb{E}[(\widehat{Y}-Y)^2]\) and fit \(\log \widehat{v}_0(x_k)\approx \log c - \beta\log x_k\) over a range where scaling is stable.

\paragraph{Communication fidelity (\(\gamma(m)\)) or distortion (\(\sigma_c^2(m)\)).}
Communication is calibrated with a purpose-built transmission task that isolates encoding/decoding from problem solving, using the same message schema and truncation rules as the target pipeline.

For binary tasks, sample \(Y\in\{-1,+1\}\) uniformly, ask a sender to encode \(Y\) into at most \(m\) tokens, and ask a receiver to decode \(\widehat{Y}\) from the message alone.
The empirical flip rate \(\widehat{\epsilon}(m)=\Pr(\widehat{Y}\neq Y)\) yields \(\widehat{\gamma}(m)=1-2\widehat{\epsilon}(m)\).
For continuous tasks, sample \(Y\in[0,1]\), transmit under an \(m\)-token constraint, decode \(\widehat{Y}\), and estimate \(\widehat{\sigma}_c^2(m)=\mathbb{E}[(\widehat{Y}-Y)^2]\).

\paragraph{Shared-failure correlation (\(\rho\)).}
Correlation should be measured at the level where aggregation occurs (siblings), since that is the dependence that limits cancellation.
On tasks \(t=1,\dots,T\) with ground truth \(Y(t)\), run \(M\) leaf agents (different randomness and/or prompts).
For binary tasks define signed correctness \(S_i(t)=\widehat{Y}_i(t)Y(t)\in\{-1,+1\}\), compute pairwise sample correlations across tasks
\[
\widehat{\rho}_{ij} := \widehat{\mathrm{Corr}}_{t\in[T]}\big(S_i(t),S_j(t)\big),
\qquad 1\le i<j\le M,
\]
and aggregate
\[
\widehat{\rho} := \frac{2}{M(M-1)}\sum_{1\le i<j\le M}\widehat{\rho}_{ij}.
\]
For continuous tasks, use residuals \(E_i(t)=\widehat{Y}_i(t)-Y(t)\) and compute residual correlations analogously.
In practice it can be informative to stratify tasks by difficulty (e.g., baseline accuracy bins), since \(\rho\) may increase on harder tasks.

\paragraph{Budget accounting and effective context.}
Finally, record realized costs: tokens generated and tokens read (including system prompts, tool outputs, and formatting overhead), and verify the effective context window \(W\) after templating.
This prevents a common pitfall: attributing gains to ``topology'' when the true driver is uneven budget usage.

\begin{figure}[t]
\centering
\definecolor{Blue}{RGB}{30,100,170}
\definecolor{Orange}{RGB}{220,165,50}
\definecolor{Gray}{RGB}{100,120,145}

\begin{tikzpicture}[x=1.0cm, y=1.0cm]

  \begin{scope}[shift={(0,0)}]
    \node[anchor=south, font=\small\bfseries, align=center] at (1.7,3.2) {(a) Single-agent\\[-2pt]scaling};
    \draw[->, thick, draw=Gray] (0,0) -- (3.4,0) node[below, font=\small] {$\log x$};
    \draw[->, thick, draw=Gray] (0,0) -- (0,2.9) node[left, font=\small] {$\log \hat{\mu}(x)$};
    \draw[thick, draw=Blue] plot[smooth] coordinates {(0.2,0.35) (0.9,0.7) (1.7,1.1) (2.5,1.5) (3.1,1.8)};
    \draw[densely dashed, draw=Gray!70] (0.4,0.4) -- (3.0,1.85);
    \node[font=\scriptsize, text=Gray] at (2.1,2.35) {fit slope $\hat{\beta}$};
  \end{scope}

  \begin{scope}[shift={(5.0,0)}]
    \node[anchor=south, font=\small\bfseries, align=center] at (1.7,3.2) {(b) Communication\\[-2pt]fidelity};
    \draw[->, thick, draw=Gray] (0,0) -- (3.4,0) node[below, font=\small] {$m$ (tokens)};
    \draw[->, thick, draw=Gray] (0,0) -- (0,2.9) node[left, font=\small, align=right] {$\hat{\gamma}(m)$ or\\[-2pt]$1/\hat{\sigma}_c^2(m)$};
    \draw[thick, draw=Orange] plot[smooth] coordinates {(0.2,0.25) (0.9,0.6) (1.7,1.05) (2.5,1.55) (3.1,1.9)};
    \node[font=\scriptsize, text=Gray] at (1.8,2.35) {calibration curve};
  \end{scope}

  \begin{scope}[shift={(10.0,0)}]
    \node[anchor=south, font=\small\bfseries, align=center] at (1.7,3.2) {(c) Shared-error\\[-2pt]correlation};
    \draw[->, thick, draw=Gray] (0,0) -- (3.4,0) node[below, font=\small] {diversity};
    \draw[->, thick, draw=Gray] (0,0) -- (0,2.9) node[left, font=\small] {$\hat{\rho}$};
    \draw[thick, draw=Blue] plot[smooth] coordinates {(0.2,2.3) (0.9,2.0) (1.7,1.5) (2.5,1.0) (3.1,0.7)};
    \node[font=\scriptsize, text=Gray, align=center] at (2.5,2.3) {more diversity\\[-1pt]reduces $\rho$};
  \end{scope}

\end{tikzpicture}

\caption{Calibration experiments (schematic). Left: estimate the single-agent scaling exponent \(\beta\) by fitting a power law in a non-saturated range.
Middle: estimate communication fidelity as a function of message length \(m\) via a transmission task.
Right: estimate shared-error correlation \(\rho\) and verify that diversity interventions reduce it.}
\label{fig:calibration_schematic}
\end{figure}
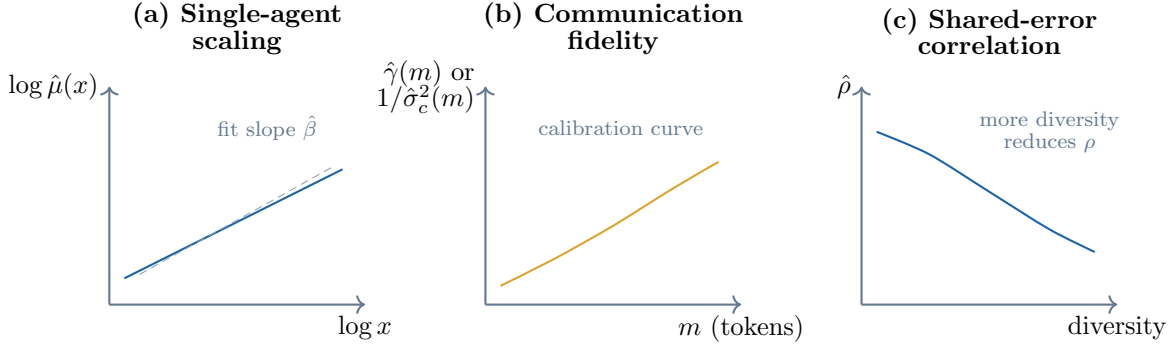

\subsection{LLM diagnostic template: context limits, correlation floors, and saturation}
\label{sec:llm_case_studies}

The synthetic setting checks internal consistency under the assumed generative models.
The more important question is whether the framework remains useful as an engineering diagnostic when agents are real LLM instances with non-Gaussian errors and prompt effects.
Our aim here is therefore qualitative: do calibrated parameters correctly identify which lever is binding and which intervention should help?

We recommend using at least one binary family with unambiguous correctness (e.g., multiple-choice, exact-answer math, pass/fail tests) and one continuous family with graded scores (e.g., rubric grading or scalar estimation with known targets), since the framework makes distinct predictions in each.
Implement single-agent, star, tree (and optionally chain) pipelines under matched realized budgets, standardizing prompts and message schemas across topologies.
For stars, feasibility is \(Nm\le W\); for trees, feasibility is \(bm\le W\), with depth and leaf count chosen using the design rules in Section~\ref{sec:design}.

In this setting, the most useful validations are diagnostic.
First, under fixed \(W\), the best feasible star should saturate in effective parallelism once \(N\) hits \(\lfloor W/m\rfloor\), so additional budget primarily benefits scale-up.
Second, measured \(\widehat{\rho}\) should help explain why naive voting sometimes fails: high \(\widehat{\rho}\) predicts limited aggregation gains even with many agents.
Third, the calibrated quantity \(\widehat{\alpha}_\rho=\widehat{\gamma}(m)(\widehat{\rho}+(1-\widehat{\rho})f_b'(0))\) should act as a feasibility check for trees: when \(\widehat{\alpha}_\rho\le 1\), deeper hierarchies should become unstable or drift toward chance; when \(\widehat{\alpha}_\rho>1\), trees should amplify and then saturate.
Finally, when performance plateaus with depth or budget, the framework predicts that further gains come primarily from improving communication fidelity (longer messages, better schemas) or reducing shared failures (diversity, verification), rather than adding depth.

\paragraph{Diagnostic interventions (compact ablations).}
Three inexpensive interventions connect the theory's parameters to practice under matched budgets: varying message length \(m\) (trading off \(\gamma(m)\) against coordination cost), injecting diversity to reduce \(\widehat{\rho}\) (e.g., prompt, model, temperature, tool-policy variation or explicit verification roles), and swapping topology (single, star, chain, tree) under fixed schemas.
When these interventions move \(\widehat{\gamma}(m)\) or \(\widehat{\rho}\) in the expected direction, the phase-diagram predictions provide testable hypotheses about whether amplification becomes feasible and whether scale-out can outpace scale-up before saturation.

\paragraph{Logging and reproducibility.}
For each run, log realized tokens per node, realized message lengths, and outputs at each level (or hashes if privacy is needed).
These logs enable post hoc estimation of \(\rho\), auditing of budget fairness, and direct diagnosis of failure modes such as repeated identical errors.

These templates are meant to support qualitative diagnostics: whether context is binding, whether measured correlation suggests limited gains, and whether the calibrated \(\widehat{\alpha}_\rho\) places a hierarchy in the amplifying or collapsing regime.
In that sense, the phase diagrams and design curves from Sections~\ref{sec:phase_diagrams}--\ref{sec:design} provide concrete hypotheses about which organization should help under matched budgets and which lever (communication fidelity versus shared-failure reduction) is most likely to matter when it does not.
Establishing quantitative predictiveness in realistic deployments, where prompts, tools, and dependence structure may drift with depth and task difficulty, is an important direction for future work.

\vskip 0.2in
\bibliography{main}

\end{document}